\documentclass{article}

\usepackage{microtype}
\usepackage{graphicx}
\usepackage{subfigure}
\usepackage{booktabs} 

\usepackage{hyperref}

\usepackage[accepted]{icml2018}

\usepackage{amsfonts}
\usepackage{amsmath}
\usepackage{commath}
\usepackage{mathbbol}
\usepackage{cancel}
\usepackage{booktabs}
\usepackage{multirow}

\usepackage{times}
\usepackage{graphicx} 
\usepackage{subfigure} 

\usepackage{natbib}

\usepackage{algorithm}
\usepackage{algorithmic}

\usepackage{amsthm,amssymb}
\newtheorem{assumption}{Assumption}
\newtheorem{proposition}{Proposition}

\newtheorem{theorem}{Theorem}

\newcommand{\VV}{\mathbb{V}} 
\newcommand{\EE}{\mathbb{E}} 
\newcommand{\RR}{\mathbb{R}}
\newcommand{\F}{\mathcal{F}}
\newcommand{\X}{\mathcal{X}}

\newcommand{\A}{\mathcal{A}}
\newcommand{\R}{r}
\newcommand{\RV}{R}

\icmltitlerunning{More Robust Doubly Robust Off-policy Evaluation}

\begin{document} 

\twocolumn[
\icmltitle{More Robust Doubly Robust Off-policy Evaluation}



\icmlsetsymbol{equal}{*}

\begin{icmlauthorlist}
\icmlauthor{Mehrdad Farajtabar}{equal,gatech}
\icmlauthor{Yinlam Chow}{equal,dm}
\icmlauthor{Mohammad Ghavamzadeh}{dm}
\end{icmlauthorlist}

\icmlaffiliation{gatech}{Georgia Institute of Technology}
\icmlaffiliation{dm}{Google DeepMind}

\icmlcorrespondingauthor{Yinlam Chow}{yinlamchow@google.com}

\icmlkeywords{Reinforcement Learning, Contextual Bandits, Off-policy Evaluation}

\vskip 0.3in
]


\printAffiliationsAndNotice{\icmlEqualContribution} 

\begin{abstract} 
We study the problem of off-policy evaluation (OPE) in reinforcement learning (RL), where the goal is to estimate the performance of a policy from the data generated by another policy(ies). In particular, we focus on the doubly robust (DR) estimators that consist of an importance sampling (IS) component and a performance model, and utilize the low (or zero) bias of IS and low variance of the model at the same time. Although the accuracy of the model has a huge impact on the overall performance of DR, most of the work on using the DR estimators in OPE has been focused on improving the IS part, and not much on how to learn the model. In this paper, we propose alternative DR estimators, called {\em more robust doubly robust} (MRDR), that learn the model parameter by minimizing the variance of the DR estimator. We first present a formulation for learning the DR model in RL. We then derive formulas for the variance of the DR estimator in both contextual bandits and RL, such that their gradients w.r.t.~the model parameters can be estimated from the samples, and propose methods to efficiently minimize the variance. We prove that the MRDR estimators are strongly consistent and asymptotically optimal. Finally, we evaluate MRDR in bandits and RL benchmark problems, and compare its performance with the existing methods.
%
%
\end{abstract} 

\vspace{-2mm}
\section{Introduction} \label{sec:intro}

In many real-world decision-making problems, in areas such as marketing, finance, robotics, and healthcare, deploying a policy without having an accurate estimate of its performance could be costly, unethical, or even illegal. This is why the problem of {\em off-policy evaluation} (OPE) has been heavily studied in contextual bandits (e.g.,~\citealt{Dudik11DR,Swaminathan17OP}) and reinforcement learning (RL) (e.g.,~\citealt{Precup00ET,Precup01OP,Paduraru13OP,Mahmood14WI,Thomas15HCPE,Li15TM,Jiang16DR,Thomas16DE}), and some of the results have been applied to problems in marketing (e.g.,~\citealt{Li11UO,Theocharous15PA}), healthcare (e.g.,~\citealt{Murphy01MM,Hirano03EE}), and education (e.g.,~\citealt{Mandel14OP,Mandel16OE}). The goal in OPE is to estimate the performance of an {\em evaluation} policy, given a log of data generated by the {\em behavior} policy(ies). The OPE problem can also be viewed as a form of counterfactual reasoning to infer the causal effect of a new treatment from historical data (e.g.,~\citealt{Bottou13CR,Shalit17EI,Louizos17CE}). 

Three different approaches to OPE in RL can be identified in the literature.  
{\bf 1) Direct Method (DM)} which learns a model of the system and then uses it to estimate the performance of the evaluation policy. This approach often has low variance but its bias depends on how well the selected function class represents the system and on whether the number of samples is sufficient to accurately learn this function class. There are two major problems with this approach: {\em (a)} Its bias cannot be easily quantified, since in general it is difficult to quantify the approximation error of a function class, and {\em (b)} It is not clear how to choose the loss function for model learning without the knowledge of the evaluation policy (or the distribution of the evaluation policies). Without this knowledge, we may select a loss function that focuses on learning the areas that are irrelevant for the evaluation policy(ies). 
{\bf 2) Importance Sampling (IS)} that uses the IS term to correct the mismatch between the distributions of the system trajectory induced by the evaluation and behavior policies. Although this approach is unbiased (under mild assumptions) in case the behavior policy is known, its variance can be very large when there is a big difference between the distributions of the evaluation and behavior policies, and grows exponentially with the horizon of the RL problem. 
{\bf 3) Doubly Robust (DR)} which is a combination of DM and IS, and can achieve the low variance of DM and no (or low) bias of IS. The DR estimator was first developed in statistics (e.g.,~\citealt{Cassel76SR,Robins94ER,Robins95SP,Bang05DR}) to estimate from incomplete data with the property that is unbiased when either of its DM or IS estimators is correct. It was brought to our community, first in contextual bandits by~\citet{Dudik11DR} and then in RL by~\citet{Jiang16DR}.~\citet{Thomas16DE} proposed two methods to reduce the variance of DR, with the cost of introducing a bias, one to select a low variance IS estimator, namely weighted IS (WIS), and one to blend DM and IS together (instead of simply combining them as in the standard DR approach) in a way to minimize the mean squared error (MSE). 

In this paper, we propose to reduce the variance of DR in bandits and RL by designing the loss function used to learn the model in the DM part of the estimator. The main idea of our estimator, called {\em more robust doubly robust} (MRDR), is to learn the parameters of the DM model by minimizing the variance of the DR estimator. This idea has been investigated in statistics in the context of regression when the labels of a subset of samples are randomly missing~\citep{Cao09IE}. We first present a novel formulation for the DM part of the DR estimator in RL. We then derive formulas for the variance of the DR estimator in both bandits and RL in a way that its gradient w.r.t.~the model parameters can be estimated from the samples. Note that the DR variances reported for bandits~\citep{Dudik11DR} and RL~\citep{Jiang16DR} contain the bias of the DM component, which is unknown. We then propose methods to efficiently minimize the variance in both bandits and RL. Furthermore, we prove that the MRDR estimator is strongly consistent and asymptotically optimal. Finally, we evaluate the MRDR estimator in bandits and RL benchmark problems, and compare its performance with DM, IS, and DR approaches.

\vspace{-2mm}
\section{Preliminaries}\label{sec:problem}

In this paper, we consider the reinforcement learning (RL) problem in which the agent's interaction with the system is modeled as a Markov decision process (MDP). Note that the contextual bandit problem is a special case with horizon-$1$ decision-making. In this section, we first define MDPs and the relevant quantities that we are going to use throughout the paper, and then define the off-policy evaluation problem in RL, which is the main topic of this work.

\subsection{Markov Decision Processes}
\label{subsec:reinforcement_learning}

A MDP is a tuple $\langle\X,\A,P_r,P,P_0,\gamma\rangle$, where $\X$ and $\A$ are the state and action spaces, $P_r(x,a)$ is the distribution of the bounded random variable $r(x,a)\in[0,R_{\max}]$ of the immediate reward of taking action $a$ in state $x$, $P(\cdot|x,a)$ is the transition probability distribution, $P_0:\X\rightarrow[0,1]$ is the initial state distribution, and $\gamma\in[0,1)$ is the discounting factor. A (stationary) policy $\pi : \X \times \A \rightarrow [0, 1]$ is a stochastic mapping from states to actions, with $\pi(a|x)$ being the probability of taking action $a$ in state $x$. We denote by $P^\pi$ the state transition of the Markov chain induced by policy $\pi$, i.e.,~\begin{small}$P^\pi(x_{t+1}|x_t)=\sum_{a\in\mathcal{A}}\pi(a|x_t)P(x_{t+1}|x_t,a)$\end{small}.

We denote by $\xi = (x_0, a_0, \R_{0},\ldots, x_{T-1}, a_{T-1}, \R_{T-1}, x_{T})$ a $T$-step trajectory generated by policy $\pi$, and by $R_{0:T-1}(\xi)=\sum_{t=0}^{T-1}\gamma^tr_t$ the return of trajectory $\xi$. Note that in $\xi$, $x_0\sim P_0$, and $\forall t\in\{1,\ldots,T-1\}$, $a_t\sim\pi(\cdot|x_t)$, $x_{t+1}\sim P(\cdot|x_t,a_t)$, and $r_t\sim P_r(\cdot|x_t,a_t)$. These distributions together define $P^\pi_\xi$, i.e.,~the distribution of trajectory $\xi$. We evaluate a policy $\pi$ by the expectation of the return of the $T$-step trajectories it generates, i.e.,~$\rho_T^\pi=\mathbb{E}_{\xi\sim P^\pi_\xi}\big[R_{0:T-1}(\xi)\big]$. If we set $T$ to be of order $O\big(1/(1-\gamma)\big)$, then $\rho_T^\pi$ would be a good approximation of the infinite-horizon performance $\rho_\infty^\pi$. Throughout the paper, we assume that $T$ has been selected such that $\rho_T^\pi\approx\rho_\infty^\pi$, and thus, we refer to $\rho^\pi=\rho_T^\pi$ as the performance of policy $\pi$. We further define the value (action-value) function of a policy $\pi$ at each state $x$ (state-action pair $(x,a)$), denoted by $V^\pi(x)$ ($Q^\pi(x,a)$), as the expectation of the return of a $T$-step trajectory generated by starting at state $x$ (state-action pair $(x,a)$), and then following policy $\pi$. Note that $\rho^\pi=\mathbb{E}_{x\sim P_0}\big[V^\pi(x)\big]$.

Note that the contextual bandit setting is a special case of the setting described above, where $T=1$, and as a result, the context is sampled from $P_0$ and there is no dynamic $P$.

\subsection{Off-policy Evaluation Problem}
\label{subsec:ope}

The off-policy evaluation (OPE) problem is when we are given a set of $T$-step trajectories $\mathcal{D}=\{\xi^{(i)}\}_{i=1}^n$ {\em independently} generated by the {\em behavior} policy $\pi_b$,\footnote{The results of this paper can be easily extended to the case that the trajectories are generated by multiple behavior policies.} and the goal is to have a good estimate of the performance of the {\em evaluation} policy $\pi_e$. We consider the estimator $\hat{\rho}^{\pi_e}$ good if it has low mean square error (MSE), i.e.,
\begin{equation}\label{eq:MSE}
\text{MSE}(\rho^{\pi_e},\hat{\rho}^{\pi_e})\stackrel{\triangle}{=}\mathbb{E}_{P^{\pi_b}_\xi}\big[(\rho^{\pi_e}-\hat{\rho}^{\pi_e})^2\big].
\end{equation} 
We make the following standard regularity assumption: 
\begin{assumption}[Absolute Continuity]
\label{assumption:abs_cont}
For all state-action pairs $(x,a)\in\X\times\A$, if $\pi_b(a|x) = 0$ then $\pi_e(a|x) = 0$.
\end{assumption}
In order to quantify the mismatch between the behavior and evaluation policies in generating a trajectory, we define {\em cumulative importance ratio} as follows. For each $T$-step trajectory $\xi\in\mathcal{D}$, the {\em cumulative importance ratio} from time step $t_1$ to time step $t_2$, where both $t_1$ and $t_2$ are in $\{0,\ldots,T\}$, is $\omega_{t_1,t_2}=1$ if $t_1>t_2$, and is $\omega_{t_1,t_2}=\prod_{\tau=t_1}^{t_2}\frac{\pi_e(a_\tau|x_\tau)}{\pi_b(a_\tau|x_\tau)}$, otherwise. In case the behavior policy $\pi_b$ is {\em unknown}, we define $\widehat{\omega}_{t_1,t_2}$ exactly as $\omega_{t_1,t_2}$, with $\pi_b$ replaced by its approximation $\widehat{\pi}_b$. Under Assumption~\ref{assumption:abs_cont}, it is easy to see that $\rho^{\pi_e}=\mathbb{E}_{P_\xi^{\pi_e}}[\sum_{t=0}^{T-1}\gamma^tr_t]=\mathbb{E}_{P_\xi^{\pi_b}}[\sum_{t=0}^{T-1}\gamma^t\omega_{0:t}r_t]$. Similar equalities hold for the value and action-value functions of $\pi_e$, i.e.,~$V^{\pi_e}(x)=\mathbb{E}_{P^{\pi_e}_\xi}[\sum_{t=0}^{T-1}\gamma^tr_t|x_0=x]=\mathbb{E}_{P_\xi^{\pi_b}}[\sum_{t=0}^{T-1}\gamma^t\omega_{0:t}r_t|x_0=x]$ and $Q^{\pi_e}(x,a)=\mathbb{E}_{P_\xi^{\pi_e}}[\sum_{t=0}^{T-1}\gamma^tr_t|x_0=x,a_0=a]=\mathbb{E}_{P_\xi^{\pi_b}}[\sum_{t=0}^{T-1}\gamma^t\omega_{0:t}r_t|x_0=x,a_0=a]$.

\vspace{-2mm}
\section{Existing Approaches to OPE}\label{sec:exist_method}
The objective of MRDR is to learn the model part of a DR estimator by minimizing its variance.
MRDR is a variation of DR with a DM loss function derived from minimizing the DR's variance and is built on the top of IS and DM.
Therefore, before stating our main results in Section~\ref{sec:mdr}, we first provide a brief overview of these popular approaches.


\vspace{-2mm}
\subsection{Direct Estimators}
\label{subsec:DM}

The idea of the direct method (DM) is to first learn a model of the system and then use it to estimate the performance of the evaluation policy $\pi_e$. In the case of bandits, this model is the mean reward of each pair of context and arm, and in RL it is either the mean reward $r(x,a)$ and state transition $P(\cdot|x,a)$, or the value (action-value) $V(x)$ ($Q(x,a)$) function. In either case, if we select a good representation for the quantities that need to be learned, and our dataset\footnote{Note that we shall use separate datasets for learning the model in DM and evaluating the policy.} contains sufficient number of the states and actions relevant to the evaluation of $\pi_e$, then the DM estimator has low variance and small bias, and thus, has the potential to outperform the estimators resulted from other approaches.


As mentioned in Section~\ref{sec:intro}, an important issue that has been neglected in the previous work on off-policy evaluation in RL is the loss function used in estimating the model in DM. As pointed out by~\citet{Dudik11DR}, the direct approach has a problem if the model is estimated without the knowledge of the evaluation policy. This is because the distribution of the states and actions that are visited under the evaluation policy should be included in the loss function of the direct approach. In other words, if upon learning a model, we have no information about the evaluation policy (or the distribution of the evaluation policies), then it is not clear how to design the DM's loss function (perhaps a uniform distribution over the states and actions would be the most reasonable). Therefore, in this paper, we assume that the evaluation policy is known prior to learning the model.\footnote{Our results can be extended to the case that the distribution of the evaluation policies is known prior to learning the model.}

In their DM and DR experiments, both~\citet{Jiang16DR} and~\citet{Thomas16DE} learn the MDP model, $r(x,a)$ and $P(\cdot|x,a)$, although all the model learning discussion in~\citet{Thomas16DE} is about the reward of the evaluation policy $\pi_e$ at every step $t$ along the $T$-step trajectory, i.e.,~$r^{\pi_e}(x,t)$. More generally, in off-policy actor-critic algorithms (such as the Reactor algorithm proposed in \citealt{gruslys2017reactor}), where one can view the gradient estimation part as an off-policy value evaluation problem, the DM state-action value function model is learned by minimizing the Bellman residual in an off-policy setting \cite{precup2000eligibility, munos2016safe, geist2014off}.
However, neither of these three approaches incorporate the design of the DM loss function into the primary objective, perhaps because they consider the setting in which the model is learned independently. 


{\bf Our approach to DM in RL:} In this paper, we propose to learn $Q^{\pi_e}$, the action-value function of the evaluation policy $\pi_e$, and then use it to evaluate its performance as

\vspace{-0.15in}
\begin{small}
\begin{equation*}
\hat{\rho}_{\text{DM}}^{\pi_e} = \frac{1}{n}\sum_{i=1}^n\sum_{a\in\A}\pi_e(a|x^{(i)}_0)\widehat Q^{\pi_e}(x^{(i)}_0,a;\beta^*_n).
\end{equation*}
\end{small}
\vspace{-0.175in}

We model $Q^{\pi_e}$ using a parameterized class of functions with parameter $\beta\in\mathbb{R}^\kappa$ and learn $\beta$ by solving the following weighted MSE problem

\vspace{-0.15in}
\begin{small}
\begin{equation}
\label{eq:DM_opt_RL}
\beta^*\in\arg\min_{\beta\in\RR^\kappa}\EE_{(x,a)\sim\mu_{\pi_e}}\Big[\big(Q^{\pi_e}(x,a)-\widehat Q^{\pi_e}(x,a;\beta)\big)^2\Big],
\end{equation}
\end{small}
\vspace{-0.175in}

where $\mu_{\pi_e}$ is the $\gamma$-discounted horizon-$T$ state-action occupancy of $\pi_e$, i.e.,~$\mu_{\pi_e}(x,a)=\frac{1-\gamma}{1-\gamma^T}\sum_{t=0}^{T-1}\gamma^t\EE_{P^{\pi_e}_\xi}\big[\mathbf{1}\{x_t$ $=x,a_t=a\}\big]$ and $\mathbf{1}\{\cdot\}$ is the indicator function. Since the actions in the data set $\mathcal{D}$ are generated by $\pi_b$, we rewrite the objective function of the optimization problem~\eqref{eq:DM_opt_RL} as 

\vspace{-0.15in}
\begin{small}
\begin{equation}
\label{eq:DM_opt_RL2}
\sum_{t=0}^{T-1}\gamma^t\EE_{P^{\pi_b}_\xi}\Big[\omega_{0:t}\big(\bar{R}_{t:T-1}(\xi)-\widehat Q^{\pi_e}(x_t,a_t;\beta)\big)^2\Big],
\end{equation}
\end{small}
\vspace{-0.15in}

where $\bar{R}_{t:T-1}(\xi)=\sum_{\tau=t}^{T-1}\gamma^{\tau-t} \omega_{t+1:\tau}\R(x_\tau,a_\tau)$ is the Monte Carlo estimate of $Q^{\pi_e}(x_t,a_t)$. The proof of the equivalence of the objective functions~\eqref{eq:DM_opt_RL} and~\eqref{eq:DM_opt_RL2} can be found in Appendix~\ref{sec:proofs-DM}. We obtain $\beta^*_n$ by solving the sample average approximation (SAA) of~\eqref{eq:DM_opt_RL2}, i.e.,

\vspace{-0.175in}
\begin{small}
\begin{align}
\label{eq:DM_opt_RL3}
\beta^*_n\in\arg\min_{\beta\in\RR^\kappa}\sum_{t=0}^{T-1}\gamma^t\cdot\frac{1}{n}\sum_{i=1}^n\omega^{(i)}_{0:t}&\big[\bar{R}_{t:T-1}(\xi^{(i)}) \nonumber \\ 
&- \widehat Q^{\pi_e}(x^{(i)}_t,a^{(i)}_t;\beta)\big]^2.
\end{align}
\end{small}
\vspace{-0.2in}

Since the SAA estimator~\eqref{eq:DM_opt_RL3} is unbiased, for large enough $n$, $\beta^*_n\rightarrow\beta^*$ almost surely. We define the bias of our DM estimator at each state-action pair as $\Delta(x,a)=\widehat{Q}^{\pi_e}(x, a;\beta)-Q^{\pi_e}(x,a)$. Note that in contextual bandits with deterministic evaluation policy, the SAA~\eqref{eq:DM_opt_RL3} may be written as the weighted least square (WLS) problem

\vspace{-0.175in}
\begin{small}
\begin{equation}
\label{eq:DM_opt_bandit}
\beta^*_n\in\frac{1}{n}\sum_{i=1}^{n}\frac{\mathbf 1\{\pi_e(x_i) = a_i\} }{\pi_b(a_i|x_i)}\big(\R(x_i,a_i) -\widehat{Q}(x_i,a_i;\beta)\big)^2,
\end{equation}
\end{small}
\vspace{-0.175in}

with weights $1/\pi_b(a_i|x_i)$ for the actions consistent with $\pi_e$.

\subsection{Importance Sampling Estimators}
\label{subsec:IS}

Another common approach to off-policy evaluation in RL is to use importance sampling (IS) to estimate the performance of the evaluation policy, i.e.,

\vspace{-0.15in}
\begin{small}
\begin{equation}
\label{eq:IS-est}
\hat{\rho}_{\text{IS}}^{\pi_e} = \frac{1}{n}\sum_{i=1}^n\omega^{(i)}_{0:T-1}\sum_{t=0}^{T-1}\gamma^t  \R^{(i)}_t = \frac{1}{n}\sum_{i=1}^n\omega^{(i)}_{0:T-1}R_{0:T-1}^{(i)},
\end{equation}
\end{small}
\vspace{-0.15in}

where $\omega^{(i)}_{0:T-1}$ and $r_t^{(i)}$ are the cumulative importance ratio and reward at step $t$ of trajectory $\xi^{(i)}\in\mathcal{D}$, respectively, and $R_{0:T-1}^{(i)}=R_{0:T-1}(\xi^{(i)})$. Under Assumption~\ref{assumption:abs_cont}, the IS estimator~\eqref{eq:IS-est} is {\em unbiased}. 

A variant of IS that often has less variance, while still unbiased, is {\em step-wise importance sampling} (step-IS), i.e.,  

\vspace{-0.125in}
\begin{small}
\begin{equation*}
\hat{\rho}_{\text{step-IS}}^{\pi_e}=\frac{1}{n}\sum_{i=1}^n\sum_{t=0}^{T-1}\gamma^t \omega^{(i)}_{0:t} \,\,\R^{(i)}_t.
\end{equation*}
\end{small}
\vspace{-0.125in}

If the behavior policy \begin{small}$\pi_b$\end{small} is {\em unknown}, which is the case in many applications, then either \begin{small}$\pi_b$\end{small} or the importance ratio \begin{small}$\omega=\pi_e/\pi_b$\end{small} needs to be estimated, and thus, IS may no longer be unbiased. In this case, the bias of IS and step-IS are \begin{small}$\left|\mathbb{E}_{P^{\pi_e}_\xi}\left[\delta_{0:T-1}(\xi)R_{0:T-1}(\xi)\right]\right|$\end{small} and \begin{small}$\left|\sum_{t=0}^{T-1}\gamma^t\mathbb{E}_{P^{\pi_e}_\xi}\left[\delta_{0:t}(\xi)\R_t\right]\right|$\end{small}, respectively, where \begin{small}$\delta_{0:t}(\xi)=1-\lambda_{0:t}(\xi)=1-\prod_{\tau=0}^t\frac{\pi_b(a_\tau|x_\tau)}{\widehat{\pi}_b(a_\tau|x_\tau)}$\end{small}, with \begin{small}$\widehat{\pi}_b$\end{small} being our approximation of \begin{small}$\pi_b$\end{small} (see the proofs in Appendix~\ref{sec:proofs-IS}). Note that when \begin{small}$\pi_b$\end{small} is {\em known}, i.e.,~\begin{small}$\widehat{\pi}_b=\pi_b$\end{small}, we have \begin{small}$\delta_{0:t}=0$\end{small}, and the bias of both IS and step-IS would be zero.


Although the unbiasedness of IS estimators is desirable for certain applications such as safety~\citep{Thomas15HCPI}, their high variance (even in the step-wise case), which grows exponentially in the horizon $T$, restricts their applications. This is why another variant of IS, called {\em weighted importance sampling} (WIS), and particularly its step-wise version, i.e.,

\vspace{-0.15in}
\begin{small}
\begin{align*}
\hat{\rho}_{\text{WIS}}^{\pi_e} &= \sum_{i=1}^n\frac{\omega^{(i)}_{0:T-1}}{\sum_{i=1}^n\omega^{(i)}_{0:T-1}}\sum_{t=0}^{T-1}\gamma^t\R^{(i)}_t = \sum_{i=1}^n\frac{\omega^{(i)}_{0:T-1}R^{(i)}_{0:T-1}}{\sum_{i=1}^n\omega^{(i)}_{0:T-1}}, \\
\hat{\rho}_{\text{step-WIS}}^{\pi_e} &= \sum_{i=1}^n\sum_{t=0}^{T-1}\gamma^t \frac{\omega^{(i)}_{0:t} \,\, \R^{(i)}_t}{\sum_{i=1}^n\omega^{(i)}_{0:t}},
\end{align*}
\end{small}
\vspace{-0.1in}

is considered more practical, especially where being biased is not crucial. The WIS estimators are biased but consistent and have lower variance than their IS counterparts.  

\subsection{Doubly Robust Estimators}
\label{subsec:DR}

Doubly robust (DR) estimators that combine DM and IS were first developed for regression (e.g.,~\citealt{Cassel76SR}), brought to contextual bandits by~\citet{Dudik11DR}, and to RL by~\citet{Jiang16DR} and~\citet{Thomas16DE}. The DR estimator for RL is defined as

\vspace{-0.15in}
\begin{small}
\begin{align}
\label{eq:DR}
\hat{\rho}_{\text{DR}}^{\pi_e}(\beta) &= \frac{1}{n} \sum_{i=1}^n\sum_{t=0}^{T-1}\Big[\gamma^t\omega_{0:t}^{(i)}\R^{(i)}_t \\ 
&- \gamma^t\big(\omega_{0:t}^{(i)}\widehat{Q}^{\pi_e}(x^{(i)}_t,a^{(i)}_t;\beta)-\omega_{0:t-1}^{(i)}\widehat{V}^{\pi_e}(x^{(i)}_t;\beta)\big)\Big].\nonumber
\end{align}
\end{small}
\vspace{-0.225in}

Eq.~\ref{eq:DR} clearly shows that a DR estimator contains both the cumulative importance ratio $\omega$ (IS part) and the model estimates $\widehat{V}^{\pi_e}$ and $\widehat{Q}^{\pi_e}$ (DM part). Note that the IS part of the DR estimator~\eqref{eq:DR} is based on step-wise IS.~\citet{Thomas16DE} derived a DR estimator whose IS part is based on step-wise WIS. In this paper, we use step-wise IS for the IS part of our DR-based estimators, but our results can be easily extended to other IS estimators.

The bias of a DR estimator is the product of that of DM and IS, and thus, DR is unbiased whenever either IS or DM is unbiased. This is what the term ``doubly robust" refers to. The bias of the DR estimator~\eqref{eq:DR} is $|\mathbb{E}_{P^{\pi_e}_\xi}[\sum_{t=0}^{T-1}\gamma^t\lambda_{0:t-1}(\xi)\delta_t(\xi)\Delta(x_t,a_t)]|$ (see the proofs in Appendix~\ref{sec:proofs-DR}), and thus, it would be zero if either $\Delta(x_t,a_t)$ or $\delta_t(\xi)$ is zero. As discussed in Section~\ref{subsec:IS}, if $\pi_b$ is known, $\delta_t=0$ and the DR estimator~\eqref{eq:DR} is unbiased. Throughout this paper, we assume that $\pi_b$ is known, and thus, DR is unbiased as long as it uses unbiased variants of IS. However, our proposed estimator described in Section~\ref{sec:mdr} can be extended to the case that $\pi_b$ is unknown.

\vspace{-2mm}
\section{More Robust Doubly Robust Estimators}\label{sec:mdr}

In this section, we present our class of more robust doubly robust (MRDR) estimators. The main idea of MRDR is to learn the DM parameter of a DR estimator, $\beta\in\mathbb{R}^\kappa$, by minimizing its variance. In other words, MRDR is a variation of DR with a DM loss function derived from minimizing the DR's variance. As mentioned earlier, we assume that the behavior policy $\pi_b$ is known, and thus, both IS (step-IS) and DR estimators are unbiased. This means that our MRDR estimator is also unbiased, and since it is the result of minimizing the DR's variance, it has the lowest MSE among all the DR estimators. 


\subsection{MRDR Estimators for Contextual Bandits}
\label{subsec:MRDR-bandit}

Before presenting MRDR for RL, we first formulate it in the contextual bandit setting. We follow the setting of~\citet{Dudik11DR} and define the DR estimator as 

\vspace{-0.175in}
\begin{small}
\begin{equation}
\label{eq:DR-bandit}
\begin{split}
\hat{\rho}_{\text{DR}}^{\pi_e}(\beta) =\frac{1}{n} \sum_{i=1}^n \frac{\pi_e(a_i|x_i)}{\widehat\pi_b(a_i|x_i)} &\big(\R(x_i,a_i) - \widehat{Q}(x_i,a_i;\beta)\big) \\
&\quad+ \widehat{V}^{\pi_e} (x_i;\beta),
\end{split}
\end{equation}
\end{small}
\vspace{-0.1in}

where \begin{small}$\widehat{Q}(x,a;\beta)\approx Q(x,a)=\mathbb{E}_{P_r}[r(x,a)]$\end{small} and \begin{small}$\widehat{V}^{\pi_e}(x;\beta)=\mathbb{E}_{a\sim\pi_e}[\widehat{Q}(x,a;\beta)]$\end{small}. We further define the DM bias \begin{small}$\Delta(x,a)=$ $\widehat{Q}(x,a;\beta)-Q(x,a)$\end{small}, and error in learning the behavior policy \begin{small}$\delta(x,a)=1-\lambda(x,a)=1-\frac{\pi_b(a|x)}{\widehat\pi_b(a|x)}$\end{small}. Proposition~\ref{lem:tech_result_bias_variance} proves the bias and variance of DR for stochastic evaluation policy $\pi_e$. Note that the results stated in Theorems~1 and~2 in~\citet{Dudik11DR} are only for deterministic $\pi_e$.


\begin{proposition}
\label{lem:tech_result_bias_variance}
The bias and variance of the DR estimator~\eqref{eq:DR-bandit} for stochastic $\pi_e$ may be written as

\vspace{-0.15in}
\begin{small}
\begin{align*}
\text{Bias}(\hat{\rho}_{\text{DR}}^{\pi_e}) &= \left|\rho^{\pi_e} - \mathbb{E}_{P^{\pi_b}_\xi}[\hat{\rho}_{\text{DR}}^{\pi_e}]\right| = \left|\mathbb E_{P^{\pi_e}_\xi}\left[\delta(x, a) \Delta(x, a)\right]\right|,  \\
n\mathbb{V}_{P^{\pi_b}_\xi}(\hat{\rho}_{\text{DR}}^{\pi_e}) &= \EE_{P^{\pi_b}_\xi}\left[\widehat{\omega}(x,a)^2\big(\R(x,a) - Q(x,a)\big)^2\right] \\ 
&+\mathbb V_{P_0}\left(\EE_{\pi_e}\left[Q(x, a)+\delta(x,a)\Delta(x,a)\right]\right) \\ 
&+\EE_{P_0,\pi_e}\Big[\omega(x,a)\big(1-\delta(x,a)\big)^2\Delta(x,a)^2 \\ 
&-\EE_{\pi_e}\big[\big(1-\delta(x,a)\big)\Delta(x,a)\big]^2\Big].
\end{align*}
\end{small}
\vspace{-0.25in}

\end{proposition}

\begin{proof}
See Appendix~\ref{sec:proofs-MRDR-bandit}.
\end{proof}


As expected from a DR estimator, Proposition~\ref{lem:tech_result_bias_variance} shows that~\eqref{eq:DR-bandit} is unbiased if either its DM part is unbiased, $\Delta=0$, or its IS part is unbiased, $\delta=0$. When the behavior policy $\pi_b$ is known, and thus, $\delta(x,a)=0$ for all $x$ and $a$, the variance of~\eqref{eq:DR-bandit} in Proposition~\ref{lem:tech_result_bias_variance} may be written as 

\vspace{-0.175in}
\begin{small}
\begin{align}
\label{eq:DR-var-bandit2}
&n\mathbb{V}_{P^{\pi_b}_\xi}(\hat{\rho}_{\text{DR}}^{\pi_e}) = \EE_{P^{\pi_b}_\xi}\left[\omega(x,a)^2\big(\R(x,a) - Q (x,a)\big)^2\right] \\ 
&+\mathbb V_{P_0}\big[V^{\pi_e}(x)\big]+\EE_{P_0,\pi_e}\Big[\omega(x,a)\Delta(x,a)^2 - \EE_{\pi_e}\big[\Delta(x,a)\big]^2\Big]. \nonumber 
\end{align}
\end{small}
\vspace{-0.175in}



Unfortunately, the variance formulation~\eqref{eq:DR-var-bandit2} is not suitable for our MRDR method, because its derivative w.r.t.~$\beta$ contains a term $\Delta(x,a)=\widehat{Q}(x,a)-Q(x,a)$ that cannot be estimated from samples as the true expected reward $Q$ is unknown. To address this issue, we derive a new formulations of the variance in Theorem~\ref{thm:mdr_variance_cb}, whose derivative does not contain such terms. 



\begin{theorem}
\label{thm:mdr_variance_cb}
The variance of the DR estimator~\eqref{eq:DR-bandit} for stochastic $\pi_e$ may be written as the following two forms:

\vspace{-0.175in}
\begin{small}
\begin{align}
\label{eq:var1}
&n\VV_{P^{\pi_b}_\xi}(\hat{\rho}^{\pi_e}_{\text{DR}}) = \mathbb E_{P^{\pi_b}_\xi}\Big[\omega(x,a)\Big(\mathbb{E}_{\pi_e}\left[\omega(x,a')\widehat{Q}(x,a';\beta)^2\right] \nonumber \\ 
&- \widehat{V}^{\pi_e}(x;\beta)^2 - 2\R(x,a)\big(\omega(x,a)\widehat{Q}(x,a;\beta)-\widehat{V}^{\pi_e}(x;\beta)\big)\Big) \nonumber \\
&+\omega(x,a)^2\R(x,a)^2-\mathbb{E}_{\pi_e}[r(x,a)]^2\Big] + \VV_{P_0}\big(\mathbb{E}_{\pi_e}[r(x,a)]\big), \\
&n\VV_{P^{\pi_b}_\xi}(\hat{\rho}^{\pi_e}_{\text{DR}}) = \overbrace{\mathbb E_{P^{\pi_b}_\xi}\big[\omega(x,a)q_\beta(x,a,\R)^\top\Omega_{\pi_b}(x)q_\beta(x,a,\R)\big]}^{J(\beta)} \nonumber \\
&\qquad\qquad\quad\; + C,
\label{eq:var2}
\end{align}
\end{small}
\vspace{-0.25in}

where \begin{small}$\Omega_{\pi_b}(x)=\mathrm{diag}\big[1/\pi_b(a|x)\big]_{a\in\mathcal{A}}-ee^\top$\end{small} is a positive semi-definite matrix (see Proposition~\ref{lemma:psd} in Appendix~\ref{sec:proofs-MRDR-bandit} fot the proof) with \begin{small}$e=[1,\ldots,1]^\top$\end{small}; \begin{small}$q_\beta(x,a,\R)=D_{\pi_e}(x)\bar{Q}(x;\beta)-\mathbb{I}(a)r$\end{small}  a row vector with \begin{small}$D_{\pi_e}(x)=\mathrm{diag}\big[\pi_e(a|x)\big]_{a\in\mathcal{A}}$\end{small}, row vector \begin{small}$\bar{Q}(x;\beta)=\big[\widehat{Q}(x,a;\beta)\big]_{a\in\mathcal{A}}$\end{small}, and the row vector of indicator functions \begin{small}$\mathbb I(a)=\big[\mathbf 1\{a'=a\}\big]_{a'\in\mathcal A}$\end{small}; and finally \begin{small}$C=\mathbb V_{P_0}\big(\mathbb{E}_{\pi_e}[r(x,a)]\big) -\mathbb E_{P^{\pi_b}_\xi}\big[\mathbb{E}_{\pi_e}[r(x,a)]^2\big]+\mathbb E_{P^{\pi_b}_\xi}\Big[\big(1+\omega(x,a)-\frac{1}{\pi^2_b(a|x)}\big)\omega(x,a)\R(x,a)^2\Big]$\end{small}.
\end{theorem}

\begin{proof}
See Appendix~\ref{sec:proofs-MRDR-bandit}.
\end{proof}

The significance of the variance formulations of Theorem~\ref{thm:mdr_variance_cb} is {\bf 1)} the variance of the DR estimator has no dependence on the unknown term $\Delta$, and thus, its derivative w.r.t.~$\beta$ is computable, {\bf 2)} the expectation in~\eqref{eq:var2} is w.r.t.~$P^{\pi_b}_\xi$, which makes it possible to replace $J(\beta)$ with its unbiased SAA 

\vspace{-0.15in}
\begin{small}
\begin{equation*}
J_n(\beta) = \frac{1}{n}\sum_{i=1}^n\omega(x_i,a_i)q_\beta(x_i,a_i, \R_i)^\top\Omega_{\pi_b}(x_i)q_\beta(x_i,a_i, \R_i),
\end{equation*}
\end{small}
\vspace{-0.15in}

where $\mathcal{D}=\{(x_i,a_i,r_i)\}_{i=1}^n$ is the data set generated by the behavior policy $\pi_b$, such that the optimizer of $J_n(\beta)$ converges to that of $J(\beta)$ almost surely, and {\bf 3)} $J(\beta)$ in~\eqref{eq:var2} is a convex quadratic function of $q_\beta$, which in case that $\widehat{Q}(x,a;\beta)$ is smooth, makes it possible to efficiently optimize $J_n(\beta)$ with stochastic gradient descent. Moreover, when $\nabla_\beta\widehat Q(x,a;\beta)$ can be explicitly written, we can obtain $\beta^*_n\in\arg\min_\beta J_n(\beta)$, by solving the first order optimality condition $\sum_{i=1}^n\omega(x_i,a_i)$ $q_\beta(x_i,a_i,r_i)^\top\Omega_{\pi_b}(x_i)D_{\pi_e}(x_i)\nabla_\beta\bar{Q}(x_i;\beta)= 0$.

In case the evaluation policy is deterministic, the variance $n\VV_{P^{\pi_b}_\xi}(\hat{\rho}^{\pi_e}_{\text{DR}})$ in~\eqref{eq:var1} becomes

\vspace{-0.25in}
\begin{small}
\begin{align*}
&\overbrace{\EE_{P^{\pi_b}_\xi}\Big[\frac{\mathbf 1\{\pi_e(x)=a\}}{\pi_b(a|x)}\cdot\frac{1- \pi_b(a|x)}{\pi_b(a|x)} \big(\R(x,a) - \widehat{Q}(x,a;\beta)\big)^2\Big]}^{J(\beta)}\\ 
&\quad\;\;+\VV_{P_0}\big(\mathbb{E}_{\pi_e}[r(x,a)]\big).
\end{align*}
\end{small}
\vspace{-0.25in}

This form of $J(\beta)$ allows us to find the model parameter of MRDR by solving the WLS 
\begin{equation}\label{eq:J_b_deterministic}
\small
\begin{split}
\beta^*_n\in&\arg\min_\beta J_n(\beta)=\frac{1}{n}\sum_{i=1}^{n}\mathbf 1\{\pi_e(x_i) = a_i\}\cdot\\
&\quad\quad\frac{1- \pi_b(a_i|x_i)}{\pi_b(a_i|x_i)^2}\big(\R(x_i,a_i)-\widehat{Q}(x_i,a_i;\beta)\big)^2.
\end{split}
\end{equation}
Comparing this WLS with that in the DM approach in~\ref{eq:DM_opt_bandit}, we note that MRDR changes the weights from $1/\pi_b$ to $(1-\pi_b)/\pi_b^2$, and this way increases the penalty of the samples whose actions are the same as those suggested by $\pi_e$, but have low probability under $\pi_b$, and decreases the penalty of the rest of the samples. 

%
%

\vspace{-0.1in}
\subsection{MRDR Estimators for Reinforcement Learning} 
\label{subsec:MRDR-RL}

We now present our MRDR estimator for RL. We begin with the DR estimator for RL given by~\eqref{eq:DR}. Similar to the bandits case reported in Section~\ref{subsec:MRDR-bandit}, we first derive a formula for the variance of the estimator~\eqref{eq:DR}, whose derivative can be easily estimated from trajectories generated by the behavior policy. We then use this variance formulation as the objective function to find the MRDR model parameter.


\begin{theorem}
\label{thm:mdr_variance}
The variance of the DR estimator in~\eqref{eq:DR} can be written as

\vspace{-0.2in}
\begin{small}
\begin{align}
\label{eq:mdr-variance}
n\mathbb V_{P^{\pi_b}_\xi}(\hat{\rho}^{\pi_e}_{\text{DR}}) &= \sum_{t=0}^{T-1}\mathbb E_{\F_{0:t-1}}\bigg[\gamma^{2t}\omega^2_{0:t-1}\VV_{\F_{t:T-1}}\Big(\omega_t\big(\bar{\RV}_{t:T-1} \nonumber \\
&- \widehat{Q}^{\pi_e}(x_t,a_t;\beta)\big)\Big) + \widehat{V}^{\pi_e}(x_t;\beta) + C_t\bigg] \\
&+\mathbb E_{\F_{0:t}}\Big[\gamma^{2t-2}\omega^2_{0:t-1}\VV_{\F_{t+1:T-1}}(\bar{\RV}_{t:T-1} \mid \F_t)\Big], \nonumber
\end{align}
\end{small}
\vspace{-0.15in}


where \begin{small}$\F_{t_1:t_2}$\end{small} is the filtration induced by the sequence \begin{small}$\{x_{t_1},a_{t_1},r_{t_1},\ldots,x_{t_2},a_{t_2},r_{t_2}\}\sim P_\xi^{\pi_b}$, $\bar{\RV}_{t:T-1}=\R(x_t,a_t)+\gamma\sum_{\tau=t+1}^{T-1}\gamma^{\tau-(t+1)}\omega_{t+1:j}\R(x_\tau,a_\tau)$\end{small}, and \begin{small}$C_t=E_{\F_{t:T-1}}\Big[\omega_t^2\big(\bar{\RV}_{t:T-1}-\mathbb E_{\F_{t+1:T-1}}[\bar{\RV}_{t:T-1}]\big)^2-2\omega_t^2\bar{\RV}_{t:T-1}$ $\big(\bar{\RV}_{t:T-1}-\mathbb E_{\F_{t+1:T-1}}[\bar{\RV}_{t:T-1}]\big)\Big]$\end{small} is a $\beta$-independent term.
\end{theorem}

\begin{proof}
The proof is by mathematical induction and is reported in Appendix~\ref{sec:proofs-MRDR-RL}.
\end{proof}


As opposed to the DR variance reported in~\citet{Jiang16DR}, ours in \eqref{eq:mdr-variance} has no dependence on the DM bias $\Delta$, which contains the unknown term $Q^{\pi_e}$, and plus, all its expectations are over $P^{\pi_b}_\xi$. This allows us to easily compute the MRDR model parameter from the gradient of~\eqref{eq:mdr-variance}.

Let's define $\beta^*\in\arg\min_{\beta\in\mathbb R^\kappa}\VV_{P^{\pi_b}_\xi}\big(\hat{\rho}_{\text{DR}}^{\pi_e}(\beta)\big)$ as the minimizer of the DR variance. We may write $\beta^*$ using the variance formulation of Theorem~\ref{thm:mdr_variance}, and after dropping the $\beta$-independent terms, as $\beta^*\in\arg\min_{\beta\in\mathbb R^\kappa}\sum_{t=0}^{T-1} \mathbb E_{\F_{0:t-1}}\Big[\gamma^{2t}\omega^2_{0:t-1}\VV_{\F_{t}}\Big(\omega_t\big(\bar{\RV}_{t:T-1}-\widehat{Q}^{\pi_e}(x_t,a_t;\beta)\big)+\widehat{V}^{\pi_e}(x_t;\beta)\Big)\Big]$. 
%
%
Similar to the derivation of~\eqref{eq:var2} for bandits, we can show that

\vspace{-0.2in}
\begin{small}
\begin{align}
\label{eq:opt_mrdr_RL1}
\beta^*\in\arg&\min_{\beta\in\mathbb R^\kappa}J(\beta)=\sum_{t=0}^{T-1}\gamma^{2t}\mathbb E_{\F_{0:t-1}}\big[\omega_{0:t-1}^2\cdot\omega_t\cdot \\
&q_\beta(x_t,a_t,\bar{\RV}_{t:T-1})^\top\Omega_{\pi_b}(x_t)q_\beta(x_t,a_t,\bar{\RV}_{t:T-1})\big]. \nonumber
\end{align}
\end{small}
\vspace{-0.2in}

As shown in Proposition \ref{lemma:psd}, $J(\beta)$ is a quadratic convex function of $q_\beta$, which means that if the approximation $\widehat Q^{\pi_e}(\cdot,\cdot;\beta)$ is smooth in $\beta$, then this problem can be effectively solved by gradient descent. Since the expectation in~\eqref{eq:opt_mrdr_RL1} is w.r.t.~$P^{\pi_b}_\xi$, we may use the trajectories in $\mathcal{D}$ (generated by $\pi_b$), replace $J(\beta)$ with its unbiased SAA, $J_n(\beta)$, and solve it for $\beta$, i.e.,

\vspace{-0.2in}
\begin{small}
\begin{align}
\label{eq:opt_mrdr_RL2}
\beta^*_n\in&\arg\min_{\beta\in\mathbb R^\kappa}J_n(\beta)=\sum_{i=1}^n\sum_{t=0}^{T-1}\gamma^{2t}(\omega^{(i)}_{0:t-1})^2\cdot\omega_t^{(i)}\cdot \\
&q_\beta(x^{(i)}_t,a^{(i)}_t,\bar{\RV}^{(i)}_{t:T-1})^\top\Omega_{\pi_b}(x^{(i)}_t)q_\beta(x^{(i)}_t,a^{(i)}_t,\bar{\RV}^{(i)}_{t:T-1}) \nonumber.
\end{align}
\end{small}
\vspace{-0.2in}

Since $J_n(\beta)$ is strongly consistent, $\beta^*_n\rightarrow\beta^*$ almost surely. If we can explicitly write $\nabla_\beta\widehat Q(x,a;\beta)$, then $\beta^*_n$ is the solution of equation $0=\sum_{i=1}^n\sum_{t=0}^{T-1}\gamma^{2t}(\omega^{(i)}_{0:t-1})^2\omega_t^{(i)}$ $q_\beta(x^{(i)}_t,a^{(i)}_t,\bar{\RV}^{(i)}_{t:T-1})^\top\Omega_{\pi_b}(x^{(i)}_t)D_{\pi_e}(x^{(i)}_t)\nabla_\beta\bar{Q}(x^{(i)}_t;\beta)$ . 

In case the evaluation policy is deterministic, we can further simplify $J_n(\beta)$ and derive the model parameter for MRDR by solving the following WLS problem:

\vspace{-0.2in}
\begin{small}
\begin{align}
\label{eq:0000}
J_n(\beta)&=\frac{1}{n}\sum_{i=1}^{n}\sum_{t=0}^{T-1}\gamma^{2t}(\omega^{(i)}_{0:t-1})^2\omega_t^{(i)}\mathbf 1\{\pi_e(x^{(i)}_t) = a^{(i)}_t\} \nonumber \\
&\frac{1- \pi_b(a^{(i)}_t|x^{(i)}_t)}{\pi_b(a^{(i)}_t|x^{(i)}_t)^2}\big(\bar{\RV}^{(i)}_{t:T-1}-\widehat{Q}^{\pi_e}(x^{(i)}_t,a^{(i)}_t;\beta)\big)^2. 
\end{align}
\end{small}
\vspace{-0.2in}

The intuition behind the weights in WLS~\eqref{eq:0000} is {\bf 1)} to adjust the difference between the occupancy measures of the behavior and evaluation policies, and {\bf 2)} to increase the penalty of the policy discrepancy term $\mathbf 1\{\pi_e(x_t) = a_t\}$. 

\subsection{Other Properties of the MRDR Estimators}
\label{subsec:MRDR-asy}
\begin{paragraph}
{\bf Strong Consistency} Similar to the analysis in~\citet{Thomas16DE} for weighted DR, we prove (in Appendix~\ref{sec:proofs-MRDR-asy}) that the MRDR estimators are strongly consistent, i.e.,
~$\lim_{n\rightarrow\infty}\hat{\rho}^{\pi_e}_{\text{MRDR},n}(\beta^*_n)= \rho^{\pi_e}$ almost surely. This implies that MRDR is a {\em well-posed} OPE estimator.  
\end{paragraph}

\vspace{-0.1in}
\begin{paragraph}
{\bf Asymptotic Optimality} The MRDR estimator, by construction, has the lowest variance among the DR estimators of the form~\eqref{eq:DR}. On the other hand, the semi-parametric theory in multivariate regression~\cite{Robins94ER} states that without extra assumption on the data distribution, the class of {\em unbiased}, {\em consistent} and {\em asymptotically normal} OPE estimators is asymptotically equivalent to the DR estimators in~\eqref{eq:DR}. Utilizing this result, we can show that the MRDR estimators are asymptotically optimal (i.e.,~have minimum variance) in this class of estimators.
\end{paragraph}



\vspace{-0.1in}
\begin{paragraph}
{\bf MRDR Extensions} Similar to~\citet{Thomas16DE}, we can derive the {\em weighted} MRDR estimator by replacing the IS part of the MRDR estimator in~\eqref{eq:DR} with (per-step) weighted importance sampling. This introduces bias, but potentially reduces its variance, and thus, its MSE. 

Throughout the paper, we assumed that the data has been generated by a single behavior policy. We can extend our MRDR results to the case that there are more than one behavior policy by replacing the IS part of our estimator with {\em fused importance sampling}~\citep{Peshkin02LS}.
\end{paragraph}



\vspace{-0.1in}
\section{Experiments}\label{sec:experiment}
In this section, we demonstrate the effectiveness of the proposed MRDR estimation by comparing it with other state-of-the art methods from Section \ref{sec:exist_method} on both contextual bandit and RL benchmark problems.

\subsection{Contextual Bandit} 
\vspace{-0.1in}

Using the $9$ benchmark experiments described in \citet{Dudik11DR}, we evaluate the OPE algorithms using the standard classification data-set from the UCI repository. Here we follow the same procedure of transforming a classification data-set into a contextual bandit dataset. For the sake of brevity, detailed descriptions of the experimental setup will be deferred to the appendix.

Given a deterministic policy $\pi$, which is a logistic regression model trained by the classification data-set, we discuss three methods of transforming it into stochastic policies. The first one, which is known as \emph{friendly softening}, constructs a stochastic policy with the following smoothing procedure: Given two constants $\alpha$ and $\beta$, and a uniform (continuous) random variable $u\in[-0.5, 0.5]$. For each $a\in\{1,\ldots,l\}$, whenever $\pi(x)=a$, the stochastic policy $\pi_{\alpha, \beta} (x)$ returns $a$ with probability $\alpha + \beta \times u$, and it returns $k$, which is a realization of the uniform (discrete) random variable in $\{1,\ldots,l\}\setminus\{a\}$ with probability $\frac{1- (\alpha + \beta \times u)}{l-1}$. The second one, which is known as \emph{adversarial softening}, constructs a stochastic policy $\pi_{\alpha, \beta} (x)$ from policy $\pi$ in a similar fashion. Whenever $\pi(x)=a$, $\pi_{\alpha, \beta} (x)$ returns $k\neq a$ with probability $\alpha + \beta \times u$, and it returns $\tilde k$, which is a realization of the uniform (discrete) random variable in $\{1,\ldots,l\}$ with probability $\frac{1- (\alpha + \beta \times u)}{l}$. The third one, which is the \emph{neutral policy}, is a uniformly random policy. We will use these methods to construct behavior and evaluation policies. Table~\ref{table:policies} summarizes their specifications. 

Here we compare the MRDR method with the direct method (DM), the importance sampling (IS) method and two doubly robust (DR) estimators. The model parameter of the DM estimator is obtained by solving the SAA of the following problem: $\beta_{\text{DM}}\in\arg\min_{\beta\in\RR^\kappa}\EE_{(x,a)\sim  P^{\pi_b}_\xi }[(Q^{\pi_e}(x,a)-\widehat Q^{\pi_e}(x,a;\beta))^2]$, which means all samples are weighted according to data, without consideration of the visiting distribution induced by the evaluation policy. The model parameters of the DR estimator is optimized based on the DM methodologies described in~\eqref{eq:DM_opt_RL}. Besides the standard DR estimator we also include another alternative that is known as the DR0, which heuristically uses the model parameter from the vanilla DM method (which is called DM0, and it assigns uniform weights over samples). 

Below are results over the five behavior policies and five algorithms on the benchmark datasets. (Due to page limit, only the results of Vehicle, SatImage, PenDigits and Letter are included in the main paper, see appendix for the remaining results.) We evaluate the accuracy of the estimation via root mean squares error (MSE): $\sqrt{\sum_{j=1}^{N} (\hat \rho^{\pi_e}_j-\rho^{\pi_e})^2/N}$, where $\widehat\rho^{\pi_e}_j$ is the estimated value from the $j$-th dataset. Furthermore, we perform a $95\%$ significance test \emph{only} on MRDR and DR, with bold numbers indicating the corresponding method outperforms its counterpart significantly. 

In the contextual bandit experiments, it's clear that in most cases the proposed MRDR estimator is superior to all alternative estimators (statistical) significantly. Similar to the results reported in \citet{Dudik11DR}, the DM method incurs much higher MSE than other methods in all of the experiments. This is potentially due to the issue of high bias in model estimation when the sample-size is small.
In general the estimation error is increasing across rows from top to bottom. This is expected due to the increasing difficulties in the OPE tasks that is accounted by the increasing mis-matches between behavior and evaluation policies. 
Although there are no theoretical justifications, in most cases the performance of DR estimators (with the DM method described in Section \ref{subsec:DM}) is better than that of DR0. This also illustrates the benefits of optimizing the model parameter based on the knowledge of trajectory distribution $P^{\pi_e}_\xi$, which is generated by the evaluation policy.

\vspace{-0.2in}
\begin{table}[H]
\centering
\scriptsize
\caption{Behavior and Evaluation Policies}
\label{table:policies}
\begin{tabular}{c|c|cc}
&Policy & $\alpha$ & $\beta$ \\ \midrule \midrule
Evaluation Policy & & 0.9 & 0 \\ \midrule
\multirow{5}{*}{Behavior Policies}  &Friendly I & 0.7 & 0.2 \\ 
&Friendly II & 0.5 & 0.2 \\ 
&Neutral & - & - \\ 
&Adversary I & 0.3 & 0.2 \\ 
&Adversary II & 0.5 & 0.2 \\ 
\end{tabular}
\end{table}
\vspace{-0.2in}

\begin{table}[H]
\vspace{-0.2in}
\centering
\scriptsize
\caption{Vehicle}
\label{table:vehicle}
\begin{tabular}{cccccccc}
Behavior Policy  &DM & IS & DR & MRDR & DR0   \\
 \midrule   \midrule
Friendly I&0.3273&0.0347&0.0217&\bf 0.0202&0.0224 \\\midrule
Friendly II&0.3499&0.0517&0.0331&\bf 0.0318&0.0356\\ \midrule
Neutral&0.4384&0.087&0.0604&\bf 0.0549&0.0722\\ \midrule
Adversary I&0.405&0.0937&0.0616&\bf 0.0516&0.0769\\ \midrule
Adversary II&0.405&0.1131&0.0712&\bf 0.0602&0.0952
\end{tabular}
\end{table}
\vspace{-0.2in}

\begin{table}[H]
\vspace{-0.2in}
\centering
\scriptsize
\caption{SatImage}
\label{table:satimage}
\begin{tabular}{cccccccc}
Behavior Policy  &DM & IS & DR & MRDR & DR0   \\
 \midrule   \midrule
 Friendly I&0.2884&0.0128&0.0071&\bf 0.0063&0.0073 \\ \midrule
Friendly II &0.3328&0.0191&0.0107&\bf 0.0087&0.0119 \\ \midrule
 Neutral&0.3848&0.0413&0.0246&\bf 0.0186& 0.0335 \\ \midrule
Adversary I &0.3963&0.0459&0.027&\bf 0.0195&0.0383 \\ \midrule
Adversary II &0.4093&0.0591&0.0364&\bf 0.0262&0.0521 
\end{tabular}
\end{table}
\vspace{-0.2in}

\begin{table}[H]
\vspace{-0.2in}
\centering
\scriptsize
\caption{PenDigits}
\label{table:pendigits}
\begin{tabular}{cccccccc}
Behavior Policy  &DM & IS & DR & MRDR & DR0   \\
 \midrule
 \midrule
 Friendly I &0.4014& 0.0103 & 0.0056 & {\bf 0.0037} & 0.0059  \\
  \midrule
 Friendly II &0.4628& 0.0159 & 0.0092 & {\bf 0.0056} & 0.0194  \\
  \midrule
 Neutral &0.564& 0.0450 & 0.0314 & {\bf 0.0138} & 0.0412  \\
  \midrule
 Adversary I &0.5861& 0.0503 & 0.0366 & {\bf 0.0172} & 0.0472  \\
  \midrule
 Adversary II &0.5641& 0.0646 & 0.0444 & {\bf 0.0188} & 0.0611  \\
\end{tabular}
\end{table}
\vspace{-0.2in}

\begin{table}[H]
\vspace{-0.2in}
\centering
\scriptsize
\caption{Letter}
\label{table:letter}
\begin{tabular}{cccccccc}
Behavior Policy  & DM & IS & DR & MRDR & DR0   \\
 \midrule   \midrule
Friendly I&0.392&0.0074&0.0056&\bf 0.0044&0.0057\\ 
\midrule
Friendly II&0.4146&0.0102&0.0077&\bf 0.0054&0.0083\\
\midrule
Neutral&0.4713&0.0467&0.0363& \bf 0.0315& 0.0456\\
\midrule
Adversary I&0.46&0.0587&0.0455&\bf 0.0385&0.0575\\
\midrule
Adversary II&0.4728&0.0714&0.055&0.0481&0.0703
\end{tabular}
\end{table}
\vspace{-0.2in}

\subsection{Reinforcement Learning }
\vspace{-0.05in}

In this section we present the experimental results of OPE in reinforcement learning.
We first test the OPE algorithms on the standard domains ModelWin, ModelFail, and $4\times 4$ Maze, with behavior and evaluation policies used in \citet{Thomas16DE}. The schematic diagram of the domains is shown in Figure \ref{fig:environments}.
To demonstrate the scalability of the proposed OPE methods, we also test the OPE algorithms on the following two domains with continuous state space: Mountain Car and Cart Pole. To construct the stochastic behavior and evaluation policies, we first compute the optimal policy using standard RL algorithms such as SARSA and $Q$-learning. Then these policies are constructed by applying friendly softening to the optimal policy with specific values of $(\alpha,\beta)$.
For both domains, the evaluation policy is constructed using $(\alpha,\beta)=(0.9,0.05)$, and the behavior policy is constructed analogously using $(\alpha,\beta)=(0.8,0.05)$. 
Detailed explanations the experimental setups can be found in the appendix.
In the following experiments we set the discounting factor to be $\gamma=1$.
 
For both the ModelFail and ModelWin domains, the number of training trajectories is set to $64$, for  Maze, Mountain Car, and Cart Pole domains this number is set to $1024$. The number of trajectories for sampling-based part of estimators varies from $32$ to $512$ for the ModelWin, ModelFail, and Cart Pole domains, and varies from $128$ to $2048$ for the Maze domain and Mountain Car domains.

\vspace{-0.125in}
\begin{figure}[H]
\centering
  \includegraphics[width=0.3\textwidth]{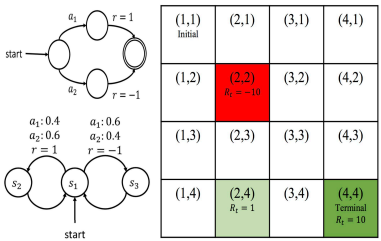} 
  \vspace{-0.225in}
  \caption{Environments from~\citet{Thomas16DE}. Top left: ModelFail; Bottom left: ModelWin; Right: Maze}
  \label{fig:environments}
\end{figure}
\vspace{-0.15in}

In all of the above experiments, we compare results of MRDR with DM, IS, DR, and DR0 estimations by their corresponding MSE values. Similarly, the bold numbers represent cases when the performance of the MRDR estimator is statistically significantly better than that of the DR estimator.
Similar to the contextual bandit setting, except for the ModelWin domain that is known to be in favor of the DM estimator \cite{Thomas16DE}, in most cases MRDR estimator has significantly lower MSE than other existing methods. 
Furthermore, when the sample size of the evaluation trajectories increases, we also observe accuracy improvements on all estimators in every experiment. Similar to the contextual bandit setting, significant performance improvement can be observed when one switches from DR0 to DR in the RL experiments.

\begin{table}[H]
\vspace{-0.175in}
\centering
\scriptsize
\caption{ModelFail}
\label{table:modelfail}
\begin{tabular}{cccccccc}
\text{Sample Size} &DM&IS& DR& MRDR& DR0 \\ \midrule \midrule
32&0.07152&1.37601&0.18461&\bf 0.1698&1.16084 \\ \midrule
64&0.07152&1.07213&0.1314&\bf 0.11405&0.9046 \\ \midrule
128&0.07152&0.752&0.09901&\bf 0.08188&0.63571 \\ \midrule
256&0.07152&0.55955&0.06565&\bf 0.05527&0.47211\\ \midrule
512&0.07152&0.39533&0.04756&\bf 0.03819&0.33391
\end{tabular}
\vspace{-0.175in}
\end{table}
\begin{table}[H]
\vspace{-0.175in}
\centering
\scriptsize
\caption{Modelwin}
\label{table:modelwin}
\begin{tabular}{cccccccc}
\text{Sample Size} &DM & IS& DR& MRDR& DR0 \\ \midrule \midrule
32&0.06182&0.78452&1.55244&\bf 1.46778&1.51858\\ \midrule
64&0.06182&1.03207&1.13856&\bf 0.98433&1.40758\\ \midrule
128&0.06182&0.90166&1.4195&\bf 1.27891&1.52634 \\ \midrule
256&0.06182&0.78507&1.03575&\bf 0.79849&1.10332 \\ \midrule
512&0.06182&0.55647&0.89655&0.66791&0.97128
\end{tabular}
\end{table}
\vspace{-0.175in}
\begin{table}[H]
\vspace{-0.175in}
\centering
\scriptsize
\caption{$4\times4$ Maze}
\label{table:maze}
\begin{tabular}{cccccccc}
\text{Sample Size} &DM  & IS& DR& MRDR& DR0 \\ \midrule \midrule
128&1.77598&6.68579&0.70465&\bf 0.57042&0.70969 \\ \midrule
256&1.77598&3.50346&0.69886&\bf 0.58871&0.70211 \\ \midrule
512&1.77598&2.64257&0.60124&0.58879&0.60338 \\ \midrule
1024&1.77598&1.45434&0.5201&\bf 0.4666&0.52148 \\ \midrule
2048&1.77598&0.89668&0.3932&\bf 0.31274&0.39425
\end{tabular}
\end{table}
\vspace{-0.175in}

\begin{table}[H]
\vspace{-0.175in}
\centering
\scriptsize
\caption{Mountain Car}
\label{table:mountain_car}
\begin{tabular}{cccccccc}
\text{Sample Size} &DM& IS& DR& MRDR& DR0 \\ \midrule \midrule	
128&17.80368&23.11318&16.14661&\bf 14.96227&19.46953 \\ \midrule
256&14.62359&14.82684&13.89212&\bf 12.48327&22.80573\\ \midrule
512&13.22012&8.26484&8.01421&7.89474&7.96849\\ \midrule
1024&10.24318&3.26843&3.03239& 3.1359&9.16269 \\ \midrule
2048&10.91577&2.50591&2.75933&\bf 2.17138&8.25527
\end{tabular}
\end{table}
\vspace{-0.175in}

\begin{table}[H]
\vspace{-0.175in}
\centering
\scriptsize
\caption{Cart Pole}
\label{table:cart_pole}
\begin{tabular}{cccccccc}
\text{Sample Size} &DM& IS& DR& MRDR& DR0 \\ \midrule \midrule						
32&3.92319&1.18213&0.34775&\bf 0.27208 &0.40567 \\ \midrule
64&3.97312&0.82658&0.27905&\bf 0.2353&0.31494 \\ \midrule
128&3.92319&0.66174&0.18793&0.16455&0.21232 \\ \midrule
256&3.82333&0.62042&0.17091&\bf 0.16012 &0.1915 \\ \midrule
512&3.80461&0.31021&0.08455&\bf 0.079&0.08946
\end{tabular}
\end{table}
\vspace{-0.175in}



\vspace{-0.25in}
\section{Conclusions}\label{sec:conclusions} 

\vspace{-0.025in}

In this paper, we proposed the class of more-robust doubly-robust (MRDR) estimators for off-policy evaluation in RL. In particular, we proposed a principled method to calculate the model in DR estimator, which aims at minimizing its variance. Furthermore, we showed that our estimator is consistent and asymptotically optimal in the class of unbiased, consistent and asymptotically normal estimators. Finally, we demonstrated the effectiveness of our MRDR estimator in bandits and RL benchmark problems. 

Future work includes extending the MRDR estimator to the cases {\bf 1)} when there are multiple behavior policies, {\bf 2)} when the action set has a combinatorial structure, e.g.,~actions are in the form of slates~\citep{Swaminathan17OP}, and {\bf 3)} when  the behavior policy is unknown.

\newpage
\bibliography{ref}
\bibliographystyle{icml2018}

\newpage
\appendix
\onecolumn

\section{Proofs of Section~\ref{subsec:DM}}\label{sec:proofs-DM}

\begin{proposition}
Solving the weighted MSE problem 

\vspace{-0.15in}
\begin{small}
\begin{equation*}
\beta^*\in\arg\min_{\beta\in\RR^\kappa}\EE_{(x,a)\sim\mu_{\pi_e}}\Big[\big(Q^{\pi_e}(x,a)-\widehat Q^{\pi_e}(x,a;\beta)\big)^2\Big],
\end{equation*}
\end{small}
\vspace{-0.175in}

where $\mu_{\pi_e}$ is the $\gamma$-discounted horizon-$T$ state-action occupancy of $\pi_e$ and $\mathbf{1}\{\cdot\}$ is the indicator function is equivalent to solving 

\vspace{-0.15in}
\begin{small}
\begin{equation*}
\beta^*\in\arg\min_{\beta\in\RR^\kappa}\sum_{t=0}^{T-1}\gamma^t\EE_{P^{\pi_b}_\xi}\Big[\omega_{0:t}\big(\bar{R}_{t:T-1}(\xi)-\widehat Q^{\pi_e}(x_t,a_t;\beta)\big)^2\Big],
\end{equation*}
\end{small}
\vspace{-0.175in}

where $\bar{R}_{t:T-1}(\xi)=\sum_{\tau=t}^{T-1}\gamma^{\tau-t} \omega_{t+1:\tau}\;\R(x_\tau,a_\tau)$.
\end{proposition}

\begin{proof}
We first expand the occupation measure $\mu_{\pi_e}$ and use the change of measures by importance sampling as follows:

\vspace{-0.15in}
\begin{small}
\begin{align*}
&\arg\min_{\beta\in\RR^\kappa}\EE_{(x,a)\sim \mu_{\pi_e}}\left[\left(Q^{\pi_e}(x,a)-\widehat Q^{\pi_e}(x,a;\beta)\right)^2\right] = \arg\min_{\beta\in\RR^\kappa}\sum_{t=0}^{T-1}\gamma^t\cdot\EE_{P^{\pi_e}_\xi}\left[\left(Q^{\pi_e}(x_t,a_t)-\widehat Q^{\pi_e}(x_t,a_t;\beta)\right)^2\right] = \\
&\arg\min_{\beta\in\RR^\kappa}\sum_{t=0}^{T-1}\gamma^t\cdot\EE_{P^{\pi_e}_\xi}\left[\left(\EE_{P,\pi_e}\left[\sum_{\tau=t}^{T-1}\gamma^{\tau-t} \R(x_\tau,a_\tau)\mid x_t,a_t\right]-\widehat Q^{\pi_e}(x_t,a_t;\beta)\right)^2\right] = \\
&\arg\min_{\beta\in\RR^\kappa}\sum_{t=0}^{T-1}\gamma^t\cdot\EE_{P^{\pi_b}_\xi}\left[\omega_{0:t}\cdot\left(\EE_{P,\pi_b}\left[\sum_{\tau=t}^{T-1}\gamma^{\tau-t}\omega_{t+1:\tau}\;\R(x_\tau,a_\tau)\mid x_t,a_t\right]-\widehat Q^{\pi_e}(x_t,a_t;\beta)\right)^2\right] \stackrel{\text{(a)}}{=} \\
&\arg\min_{\beta\in\RR^\kappa}\sum_{t=0}^{T-1}\gamma^t\cdot\EE_{P^{\pi_b}_\xi}\left[\omega_{0:t}\cdot\left(\sum_{\tau=t}^{T-1}\gamma^{\tau-t}\omega_{t+1:\tau}\;\R(x_\tau,a_\tau)-\widehat Q^{\pi_e}(x_t,a_t;\beta)\right)^2\right],
\end{align*}
\end{small}
\vspace{-0.175in}

{\bf (a)} To shorten the notations, let's define $f(X,Y)=\sum_{\tau=t}^{T-1}\gamma^{\tau-t}\omega_{t+1:\tau}\;\R(x_\tau,a_\tau)$ and $g(X;\beta)=\widehat Q^{\pi_e}(x_t,a_t;\beta)$, where $X=(x_t,a_t)$ and $Y=\{x_{\tau+1},a_{\tau+1},r_{\tau+1},\ldots,x_{T-1},a_{T-1},r_{T-1}\}$ are random variables. Using these notations, we may write

\vspace{-0.15in}
\begin{small}
\begin{align*}
&\arg\min_{\beta\in\RR^\kappa}\EE_{X,Y}\left[\big(f(X,Y)-g(X;\beta)\big)^2\right] = \arg\min_{\beta\in\RR^\kappa}\EE_{X,Y}\left[\left(f(X,Y)-\EE_{Y}[f(X,Y)|X] +\EE_{Y}[f(X,Y)|X]  - g(X;\beta)\right)^2\right] = \\
&\arg\min_{\beta\in\RR^\kappa}\EE_{X,Y}\bigg[\big(f(X,Y)-\EE_{Y}[f(X,Y)|X]\big)^2 + \big(\EE_{Y}[f(X,Y)|X] - g(X;\beta)\big)^2 \\
&\qquad\qquad\qquad\qquad\qquad\quad + 2\big(f(X,Y)-\EE_{Y}[f(X,Y)|X]\big)\cdot\big(E_{Y}[f(X,Y)|X]  - g(X;\beta)\big)\bigg] \stackrel{\text{(b)}}{=} \\
&\arg\min_{\beta\in\RR^\kappa}\EE_{X,Y}\left[\big(\EE_{Y}[f(X,Y)|X]  - g(X;\beta)\big)^2 - 2\big(f(X,Y)-\EE_{Y}[f(X,Y)|X]\big)\cdot g(X;\beta)\right] \stackrel{\text{(c)}}{=} \\
&\arg\min_{\beta\in\RR^\kappa}\EE_{X,Y}\left[(\EE_{Y}[f(X,Y)|X]  - g(X;\beta))^2\right],
\end{align*}
\end{small}
\vspace{-0.175in}

{\bf (b)} Here we drop the terms that are independent of $\beta$. 

{\bf (c)} This passage is due to the fact that $\EE_{Y}\Big[f(X,Y)-\EE_{Y}\big[f(X,Y)|X\big]\mid X\Big]=0$.
\end{proof}

\newpage
\section{Proofs of Section~\ref{subsec:IS}}\label{sec:proofs-IS}

\subsection{Bias of the IS Estimator ($\pi_b$ is unknown)}

In case the behavior policy $\pi_b$ is unknown, the IS estimator is written as 
\begin{equation}
\label{eq:IS0}
\hat{\rho}^{\pi_e}_{\text{IS}} = \frac{1}{n}\sum_{i=1}^n\widehat{\omega}^{(i)}_{0:T-1}\sum_{t=0}^{T-1}\gamma^t \R^{(i)}_t = \frac{1}{n}\sum_{i=1}^n\widehat{\omega}^{(i)}_{0:T-1}R^{(i)}_{0:T-1},
\end{equation}
where $\widehat{\omega}^{(i)}_{0:T-1}=\prod_{t=0}^{T-1}\frac{\pi_e(a^{(i)}_t | x^{(i)}_t)}{\widehat{\pi}_b(a^{(i)}_t | x^{(i)}_t)}$ is the approximate cumulative importance ratio of trajectory $\xi^{(i)}\in\mathcal{D}$ and $\widehat{\pi}_b$ is our approximation of the behavior policy $\pi_b$. Note that $\widehat{\pi}_b$ should be computed from a data set different than $\mathcal{D}=\{\xi^{(i)}\}_{i=1}^n$ that we use for our estimator.

\begin{proposition}
The bias of the IS estimator~\eqref{eq:IS0} is  
\begin{equation}
\label{eq:IS00}
\text{Bias}(\hat{\rho}^{\pi_e}_{\text{IS}}) = \left|\rho^{\pi_e} - \mathbb{E}_{P^{\pi_b}_\xi}[\hat{\rho}^{\pi_e}_{\text{IS}}]\right| = \left|\mathbb{E}_{P^{\pi_e}_\xi}\left[\delta_{0:T-1}(\xi)R_{0:T-1}(\xi)\right]\right|, 
\end{equation}
where $\xi=(x_0,a_0,r_0,\ldots,x_{T-1},a_{T-1},r_{T-1},x_T)$ is a trajectory and
\begin{equation*}
\delta_{0:T-1}(\xi) = 1- \prod_{t=0}^{T-1}\frac{\pi_b(a_t | x_t)}{\widehat{\pi}_b(a_t | x_t)}.
\end{equation*}
\end{proposition}

\begin{proof}
To prove~\eqref{eq:IS00}, we first develop the term $\mathbb{E}_{P^{\pi_b}_\xi}[\hat{\rho}^{\pi_e}_{\text{IS}}]$ as follows:
\begin{align}
\mathbb{E}_{P^{\pi_b}_\xi}[\hat{\rho}^{\pi_e}_{\text{IS}}] &= \mathbb{E}_{P^{\pi_b}_\xi}\left[\frac{1}{n}\sum_{i=1}^n\widehat{\omega}^{(i)}_{0:T-1}R^{(i)}_{0:T-1}\right] = \frac{1}{n}\sum_{i=1}^n\mathbb{E}_{P^{\pi_b}_\xi}\left[\widehat{\omega}^{(i)}_{0:T-1}R^{(i)}_{0:T-1}\right] \stackrel{\text{(a)}}{=} \mathbb{E}_{P^{\pi_b}_\xi}\left[\widehat{\omega}_{0:T-1}R_{0:T-1}\right] \nonumber \\ 
&\stackrel{(b)}{=} \mathbb{E}_{P^{\pi_e}_\xi}\left[\lambda_{0:T-1}R_{0:T-1}\right].
\label{eq:IS1}
\end{align}

{\bf (a)} This is because the trajectories $\xi^{(i)}\in\mathcal{D}$ are i.i.d. \\
{\bf (b)} We define $\lambda_{0:T-1}=\prod_{t=0}^{T-1}\frac{\pi_b(a_t | x_t)}{\widehat{\pi}_b(a_t | x_t)}$.

Given~\eqref{eq:IS1}, we may write the bias as 
\begin{equation*}
\text{Bias}(\hat{\rho}^{\pi_e}_{\text{IS}}) = \left|\rho^{\pi_e} - \mathbb{E}_{P^{\pi_b}_\xi}[\hat{\rho}^{\pi_e}_{\text{IS}}]\right| = \left|\mathbb{E}_{P^{\pi_e}_\xi}\left[R_{0:T-1}\right] - \mathbb{E}_{P^{\pi_e}_\xi}\left[\lambda_{0:T-1}R_{0:T-1}\right]\right| = \left|\mathbb{E}_{P^{\pi_e}_\xi}\left[\delta_{0:T-1}R_{0:T-1}\right]\right|,
\end{equation*}
which concludes the proof.
\end{proof}


\subsection{Bias of the step-IS Estimator($\pi_b$ is unknown)}

In case the behavior policy $\pi_b$ is unknown, the step-IS estimator is written as 
\begin{equation}
\label{eq:step-IS0}
\hat{\rho}^{\pi_e}_{\text{step-IS}} = \frac{1}{n}\sum_{i=1}^n\sum_{t=0}^{T-1}\gamma^t\widehat{\omega}^{(i)}_{0:t}\R^{(i)}_t,
\end{equation}
where $\widehat{\omega}^{(i)}_{0:t}=\prod_{\tau=0}^{t}\frac{\pi_e(a^{(i)}_\tau | x^{(i)}_\tau)}{\widehat{\pi}_b(a^{(i)}_\tau | x^{(i)}_\tau)}$ is the approximate cumulative importance ratio of trajectory $\xi^{(i)}\in\mathcal{D}$ and $\widehat{\pi}_b$ is our approximation of the behavior policy $\pi_b$. Note that $\widehat{\pi}_b$ should be computed from a data set different than $\mathcal{D}=\{\xi^{(i)}\}_{i=1}^n$ that we use for our estimator.

\begin{proposition}
The bias of the step-IS estimator~\eqref{eq:step-IS0} is  
\begin{equation}
\label{eq:step-IS00}
\text{Bias}(\hat{\rho}^{\pi_e}_{\text{step-IS}}) = \left|\rho^{\pi_e} - \mathbb{E}_{P^{\pi_b}_\xi}[\hat{\rho}^{\pi_e}_{\text{step-IS}}]\right| = \left|\sum_{t=0}^{T-1}\gamma^t\mathbb{E}_{P^{\pi_e}_\xi}\left[\delta_{0:t}(\xi)\R_t\right]\right|, 
\end{equation}
where $\xi=(x_0,a_0,r_0,\ldots,x_{T-1},a_{T-1},r_{T-1},x_T)$ is a trajectory and
\begin{equation*}
\delta_{0:t}(\xi) = 1- \prod_{\tau=0}^t\frac{\pi_b(a_\tau | x_\tau)}{\widehat{\pi}_b(a_\tau | x_\tau)}.
\end{equation*}
\end{proposition}

\begin{proof}
To prove~\eqref{eq:step-IS00}, we first develop the term $\mathbb{E}_{P^{\pi_b}_\xi}[\hat{\rho}^{\pi_e}_{\text{step-IS}}]$ as follows:
\begin{align}
\mathbb{E}_{P^{\pi_b}_\xi}[\hat{\rho}^{\pi_e}_{\text{step-IS}}] &= \mathbb{E}_{P^{\pi_b}_\xi}\left[\frac{1}{n}\sum_{i=1}^n\sum_{t=0}^{T-1}\gamma^t\widehat{\omega}^{(i)}_{0:t}\R^{(i)}_t\right] = \frac{1}{n}\sum_{i=1}^n\mathbb{E}_{P^{\pi_b}_\xi}\left[\sum_{t=0}^{T-1}\gamma^t\widehat{\omega}^{(i)}_{0:t}\R^{(i)}_t\right] \stackrel{\text{(a)}}{=} \mathbb{E}_{P^{\pi_b}_\xi}\left[\sum_{t=0}^{T-1}\gamma^t\widehat{\omega}_{0:t}\R_t\right] \nonumber \\ 
&= \sum_{t=0}^{T-1}\gamma^t\mathbb{E}_{P^{\pi_b}_\xi}\left[\widehat{\omega}_{0:t}\R_t\right] \stackrel{(b)}{=} \sum_{t=0}^{T-1}\gamma^t\mathbb{E}_{P^{\pi_e}_\xi}\left[\lambda_{0:t}\R_t\right].
\label{eq:step-IS1}
\end{align}

{\bf (a)} This is because the trajectories $\xi^{(i)}\in\mathcal{D}$ are i.i.d. \\
{\bf (b)} We define $\lambda_{0:t}=\prod_{\tau=0}^t\frac{\pi_b(a_\tau | x_\tau)}{\widehat{\pi}_b(a_\tau | x_\tau)}$.

Since step-IS is unbiased when $\pi_b$ is known, similar to~\eqref{eq:step-IS1}, we may write 
\begin{equation}
\label{eq:step-IS2}
\rho^{\pi_e} = \mathbb{E}_{P^{\pi_b}_\xi}\left[\frac{1}{n}\sum_{i=1}^n\sum_{t=0}^{T-1}\gamma^t\omega^{(i)}_{0:t}\R^{(i)}_t\right] = \sum_{t=0}^{T-1}\gamma^t\mathbb{E}_{P^{\pi_b}_\xi}\left[\omega_{0:t}\R_t\right] = \sum_{t=0}^{T-1}\gamma^t\mathbb{E}_{P^{\pi_e}_\xi}\left[\R_t\right].
\end{equation}
Given~\eqref{eq:step-IS1} and~\eqref{eq:step-IS2}, we may write the bias as 
\begin{equation*}
\text{Bias}(\hat{\rho}^{\pi_e}_{\text{step-IS}}) = \left|\rho^{\pi_e} - \mathbb{E}_{P^{\pi_b}_\xi}[\hat{\rho}^{\pi_e}_{\text{step-IS}}]\right| = \left|\sum_{t=0}^{T-1}\gamma^t\mathbb{E}_{P^{\pi_e}_\xi}\left[\R_t\right] - \sum_{t=0}^{T-1}\gamma^t\mathbb{E}_{P^{\pi_e}_\xi}\left[\lambda_{0:t}\R_t\right]\right| = \left|\sum_{t=0}^{T-1}\gamma^t\mathbb{E}_{P^{\pi_e}_\xi}\left[\delta_{0:t}\R_t\right]\right|,
\end{equation*}
which concludes the proof.
\end{proof}


%

\newpage
\section{Proofs of Section~\ref{subsec:DR}}\label{sec:proofs-DR}

\subsection{Bias of the DR Estimator ($\pi_b$ is unknown)}

When the behavior policy $\pi_b$ is unknown, the DR estimator~\eqref{eq:DR} is written as 
\begin{equation}\label{eq:DR_pi_b_unknown}
\hat{\rho}^{\pi_e}_{\text{DR}} = \frac{1}{n} \sum_{i=1}^n\sum_{t=0}^{T-1}\gamma^t\left(\widehat\omega^{(i)}_{0:t}\R^{(i)}_t - \widehat\omega_{0:t}^{(i)}\widehat{Q}^{\pi_e}(x^{(i)}_t,a^{(i)}_t) + \widehat\omega_{0:t-1}^{(i)}\widehat{V}^{\pi_e}(x^{(i)}_t)\right),
\end{equation}
where $\widehat{\omega}_{0:T-1}=\prod_{\tau=0}^{T-1}\frac{\pi_e(a_\tau | x_\tau)}{\widehat{\pi}_b(a_\tau | x_\tau)}$ is the approximate cumulative importance ratio and $\widehat{\pi}_b$ is our approximation of the behavior policy $\pi_b$.

\begin{proposition}
The bias of the DR estimator~\eqref{eq:DR_pi_b_unknown} is 
\begin{equation*}
\text{Bias}(\hat{\rho}^{\pi_e}_{\text{DR}}) = \left|\rho^{\pi_e} - \mathbb{E}_{P^{\pi_b}_\xi}[\hat{\rho}^{\pi_e}_{\text{DR}}]\right| = \left|\mathbb{E}_{P^{\pi_e}_\xi}\left[\sum_{t=0}^{T-1}\gamma^t\left(1-\delta_{0:t-1}(\xi)\right)\delta_t(\xi)\Delta^{\pi_e}(x_t,a_t)\right]\right|,
\end{equation*}
where $\xi=(x_0,a_0,r_0,\ldots,x_{T-1},a_{T-1},r_{T-1},x_T)$ is a trajectory and
\begin{equation*}
\delta_{0:t-1}(\xi) = 1 - \prod_{\tau=0}^{t-1}\frac{\pi_b(a_\tau | x_\tau)}{\widehat{\pi}_b(a_\tau | x_\tau)} \qquad\qquad , \qquad\qquad \delta_t(\xi) = 1-\frac{\pi_b(a_t|x_t)}{\widehat\pi_b(a_t|x_t)}.
\end{equation*}
\end{proposition}

\begin{proof}
We first develop the term $\mathbb{E}_{P^{\pi_b}_\xi}[\hat{\rho}^{\pi_e}_{\text{DR}}]$ as follows:
\begin{align}
\mathbb{E}_{P^{\pi_b}_\xi}[\hat{\rho}^{\pi_e}_{\text{DR}}] &= \mathbb{E}_{P^{\pi_b}_\xi}\left[\frac{1}{n} \sum_{i=1}^n\sum_{t=0}^{T-1}\gamma^t\left(\widehat\omega^{(i)}_{0:t}\R^{(i)}_t - \widehat\omega_{0:t}^{(i)}\widehat{Q}^{\pi_e}(x^{(i)}_t,a^{(i)}_t) + \widehat\omega_{0:t-1}^{(i)}\widehat{V}^{\pi_e}(x^{(i)}_t)\right)\right] \nonumber \\
&= \mathbb{E}_{P^{\pi_b}_\xi}\left[\sum_{t=0}^{T-1}\gamma^t\left(\widehat\omega_{0:t}\R_t - \widehat\omega_{0:t}\widehat{Q}^{\pi_e}(x_t,a_t) + \widehat\omega_{0:t-1}\widehat{V}^{\pi_e}(x_t)\right)\right] \nonumber \\
&\stackrel{\text{(a)}}{=} \mathbb{E}_{P^{\pi_b}_\xi}\left[\sum_{t=0}^{T-1}\gamma^t\left(\widehat\omega_{0:t}\R_t - \widehat\omega_{0:t}\left(Q^{\pi_e}(x_t,a_t) + \Delta^{\pi_e}(x_t,a_t)\right) + \widehat\omega_{0:t-1}\left(V^{\pi_e}(x_t) + \Delta^{\pi_e}(x_t)\right)\right)\right] \nonumber \\
&\stackrel{\text{(b)}}{=} \mathbb{E}_{P^{\pi_b}_\xi}\left[V^{\pi_e}(x_0)\right] + \mathbb{E}_{P^{\pi_b}_\xi}\left[\sum_{t=0}^{T-2}\gamma^t\widehat{\omega}_{0:t}\left(\R_t+\gamma V^{\pi_e}(x_{t+1})-Q^{\pi_e}(x_t,a_t)\right)\right] \nonumber \\ 
&+ \mathbb{E}_{P^{\pi_b}_\xi}\left[\gamma^{T-1}\widehat{\omega}_{0:T-1}\left(\R_{T-1} - Q^{\pi_e}(x_{T-1},a_{T-1})\right)\right] - \mathbb{E}_{P^{\pi_b}_\xi}\left[\sum_{t=0}^{T-1}\gamma^t\left(\widehat\omega_{0:t}\Delta^{\pi_e}(x_t,a_t)-\widehat\omega_{0:t-1}\Delta^{\pi_e}(x_t)\right)\right] \nonumber \\
&\stackrel{\text{(c)}}{=} \rho^{\pi_e} - \mathbb{E}_{P^{\pi_b}_\xi}\left[\sum_{t=0}^{T-1}\gamma^t\left(\widehat\omega_{0:t}\Delta^{\pi_e}(x_t,a_t)-\widehat\omega_{0:t-1}\Delta^{\pi_e}(x_t)\right)\right] \nonumber \\
&\stackrel{\text{(d)}}{=} \rho^{\pi_e} - \mathbb{E}_{P^{\pi_e}_\xi}\left[\sum_{t=0}^{T-1}\gamma^t\left(\lambda_{0:t}\Delta^{\pi_e}(x_t,a_t)-\lambda_{0:t-1}\mathbb{E}_{a_t\sim\pi_e(\cdot|x_t)}\left[\Delta^{\pi_e}(x_t,a_t)\right]\right)\right] \nonumber \\
&\stackrel{\text{(e)}}{=} \rho^{\pi_e} - \mathbb{E}_{P^{\pi_e}_\xi}\left[\sum_{t=0}^{T-1}\gamma^t\left(\lambda_{0:t}-\lambda_{0:t-1}\right)\Delta^{\pi_e}(x_t,a_t)\right] \stackrel{\text{(f)}}{=} \rho^{\pi_e} + \mathbb{E}_{P^{\pi_e}_\xi}\left[\sum_{t=0}^{T-1}\gamma^t\lambda_{0:t-1}\delta_t\Delta^{\pi_e}(x_t,a_t)\right].
\label{eq:DR1}
\end{align}

{\bf (a)} We define the bias of the DM estimator as $\Delta^{\pi_e}(x,a)=\widehat{Q}^{\pi_e}(x,a) - Q^{\pi_e}(x,a)$ and $\Delta^{\pi_e}(x)=\widehat{V}^{\pi_e}(x)-V^{\pi_e}(x)$.\footnote{Note that with abuse of notation, here we used $\Delta^{\pi_e}$ as both a function of $\mathcal{X}\times\mathcal{A}$ and a function of $\mathcal{X}$.} \\ 
{\bf (b)} Application of the telescopic sum. \\
{\bf (c)} It comes from the fact that {\em (i)} $\rho^{\pi_e}=\mathbb{E}_{x_0\sim P_0}\left[V^{\pi_e}(x_0)\right]$, {\em (ii)} $\mathbb{E}_{P^{\pi_b}_\xi}\left[\sum_{t=0}^{T-2}\gamma^t\widehat{\omega}_{0:t}\left(\R_t+\gamma V^{\pi_e}(x_{t+1})-Q^{\pi_e}(x_t,a_t)\right)\right]=0$, and {\em (iii)} $\R_{T-1} - Q^{\pi_e}(x_{T-1},a_{T-1})=0$, because $V^{\pi_e}(x_{T})=0$. \\
{\bf (d)} It comes from the definition of $\lambda_{0:t-1} = \prod_{\tau=0}^{t-1}\frac{\pi_b(a_\tau | x_\tau)}{\widehat{\pi}_b(a_\tau | x_\tau)}$ and the fact that $\Delta^{\pi_e}(x_t)=\mathbb{E}_{a_t\sim\pi_e(\cdot|x_t)}\left[\Delta^{\pi_e}(x_t,a_t)\right]$. \\
{\bf (e)} It comes from the fact that $\mathbb{E}_{P^{\pi_e}_\xi}\left[\lambda_{0:t-1}\mathbb{E}_{a_t\sim\pi_e(\cdot|x_t)}\left[\Delta^{\pi_e}(x_t,a_t)\right]\right]=\mathbb{E}_{P^{\pi_e}_\xi}\left[\lambda_{0:t-1}\Delta^{\pi_e}(x_t,a_t)\right]$. \\
{\bf (f)} It comes from the definition of $\delta_t$.

Finally,~\ref{eq:DR1} concludes the proof.
\end{proof}

\newpage
\section{Proofs of Section~\ref{subsec:MRDR-bandit}}\label{sec:proofs-MRDR-bandit}

\subsection{Proof of Proposition~\ref{lem:tech_result_bias_variance}} 

\begin{paragraph}
{\bf Proposition 1.} {\em The bias and variance of the DR estimator~\eqref{eq:DR-bandit} for stochastic $\pi_e$ may be written as}

\vspace{-0.15in}
\begin{small}
\begin{align*}
\text{Bias}(\hat{\rho}_{\text{DR}}^{\pi_e}) &= \left|\rho^{\pi_e} - \mathbb{E}_{P^{\pi_b}_\xi}[\hat{\rho}_{\text{DR}}^{\pi_e}]\right| = \left|\mathbb E_{P^{\pi_e}_\xi}\left[\delta(x, a) \Delta(x, a)\right]\right|,  \\
n\mathbb{V}_{P^{\pi_b}_\xi}(\hat{\rho}_{\text{DR}}^{\pi_e}) &= \EE_{P^{\pi_b}_\xi}\left[\widehat{\omega}(x,a)^2\big(\R(x,a) - Q(x,a)\big)^2\right] +\mathbb V_{P_0}\left(\EE_{\pi_e}\left[Q(x, a)+\delta(x,a)\Delta(x,a)\right]\right) \\ 
&+\EE_{P_0,\pi_e}\Big[\omega(x,a)\Big(\big(1-\delta(x,a)\big)\Delta(x,a)\Big)^2 - \EE_{\pi_e}\big[\big(1-\delta(x,a)\big)\Delta(x,a)\big]^2\Big].
\end{align*}
\end{small}
\vspace{-0.15in}

\end{paragraph}

\begin{proof}
{\bf Bias:} For the bias of the estimator, note that 

\vspace{-0.15in}
\begin{small}
\begin{align*}
\text{Bias}(\hat{\rho}_{\text{DR}}^{\pi_e}) &= \left|\rho^{\pi_e} - \mathbb{E}_{P^{\pi_b}_\xi}[\hat{\rho}_{\text{DR}}^{\pi_e}]\right| = \left|\rho^{\pi_e} - \mathbb E_{P^{\pi_b}_\xi}\left[\frac{\pi_e(a|x)}{\widehat\pi_b(a|x)}\left(\R(x,a) - \widehat{Q}(x,a;\beta)\right) + \widehat{V}^{\pi_e} (x;\beta)\right]\right| \\
&= \left|\rho^{\pi_e} - \mathbb E_{P^{\pi_e}_\xi}\left[\frac{\pi_b(a|x)}{\widehat\pi_b(a|x)}\left(Q(x, a)-\widehat{Q}(x,a;\beta)\right)\right] - \mathbb E_{P_0}\left[\widehat{V}^{\pi_e} (x;\beta)\right]\right| \\
&= \left|\rho^{\pi_e} - \mathbb E_{P^{\pi_e}_\xi}\left[\left(1-\delta(x,a)\right)\left(Q(x, a)-\widehat{Q}(x,a;\beta)\right)\right] - \mathbb E_{P_0}\left[\widehat{V}^{\pi_e} (x;\beta)\right]\right| \\
&= \left|\rho^{\pi_e} - \mathbb E_{P^{\pi_e}_\xi}\left[\delta(x, a) \Delta(x, a)\right] - \underbrace{\mathbb E_{P^{\pi_e}_\xi}\left[Q(x, a)\right]}_{\rho^{\pi_e}} + \underbrace{\mathbb E_{P^{\pi_e}_\xi}\left[\widehat{Q}(x,a;\beta)\right] - \mathbb E_{P_0}\left[\widehat{V}^{\pi_e} (x;\beta)\right]}_{0}\right| \\
&= \left|\mathbb E_{P^{\pi_e}_\xi}\left[\delta(x, a) \Delta(x, a)\right]\right|.
\end{align*}
\end{small}
\vspace{-0.15in}

{\bf Variance:} Before proving the variance, from the proof of bias, we notice that 

\vspace{-0.15in}
\begin{small}
\begin{equation*}
\mathbb{E}_{P^{\pi_b}_\xi}\left[\hat{\rho}_{\text{DR}}^{\pi_e}\right] = \mathbb{E}_{P^{\pi_e}_\xi}\left[Q(x,a)+\delta(x,a)\Delta(x,a)\right].
\end{equation*}
\end{small}
\vspace{-0.15in}

For the variance of the estimator, note that

\vspace{-0.15in}
\begin{small}
\begin{align*}
n\mathbb{V}_{P^{\pi_b}_\xi}(\hat{\rho}_{\text{DR}}^{\pi_e}) &= \mathbb{V}_{P^{\pi_b}_\xi}\left(\frac{\pi_e(a|x)}{\widehat\pi_b(a|x)}\left(\R(x,a) - \widehat{Q}(x,a;\beta)\right) + \widehat{V}^{\pi_e} (x;\beta)\right) \\ 
&= \mathbb{E}_{P^{\pi_b}_\xi}\left[\left(\frac{\pi_e(a|x)}{\widehat\pi_b(a|x)}\left(\R(x,a) - \widehat{Q}(x,a;\beta)\right) + \widehat{V}^{\pi_e} (x;\beta)\right)^2\right] - \left(\mathbb{E}_{P^{\pi_e}_\xi}\left[Q(x,a)+\delta(x,a)\Delta(x,a)\right]\right)^2 \\
&= \EE_{P^{\pi_b}_\xi}\left[\left( \frac{\pi_e(a|x)}{\widehat\pi_b(a|x)}\left(\R(x,a) - Q(x,a)-\Delta(x,a)\right) + \EE_{\pi_e}\left[Q(x,a)+\Delta(x,a)\right]\right)^2\right] \\ 
&- \left(\mathbb{E}_{P^{\pi_e}_\xi}\left[Q(x,a)+\delta(x,a)\Delta(x,a)\right]\right)^2 \\
&\stackrel{\text{(a)}}{=} \EE_{P^{\pi_b}_\xi}\left[\left(\frac{\pi_e(a|x)}{\widehat\pi_b(a|x)}\left(\R(x,a) - Q(x,a)\right)\right)^2+\left( \frac{\pi_e(a|x)}{\widehat\pi_b(a|x)} \Delta(x,a)\right)^2+\left(\EE_{\pi_e}\left[Q(x,a)+\Delta(x,a)\right]\right)^2\right] \\
&- 2\EE_{P^{\pi_b}_\xi}\left[\frac{\pi_e(a|x)}{\widehat\pi_b(a|x)}\Delta(x,a)\EE_{\pi_e}\left[Q(x,a)+\Delta(x,a)\right]\right] - \left(\mathbb{E}_{P^{\pi_e}_\xi}\left[Q(x,a)+\delta(x,a)\Delta(x,a)\right]\right)^2 \\
&- 2\underbrace{\EE_{P^{\pi_b}_\xi}\left[\frac{\pi_e(a|x)}{\widehat\pi_b(a|x)}\Delta(x,a)\left(r(x,a)-Q(x,a)\right)\right]}_{0} + 2\underbrace{\EE_{P^{\pi_b}_\xi}\left[\frac{\pi_e(a|x)}{\widehat\pi_b(a|x)}\left(r(x,a)-Q(x,a)\right)\EE_{\pi_e}\left[Q(x,a)+\Delta(x,a)\right]\right]}_{0}
\end{align*}
\end{small}
\vspace{-0.15in}

\vspace{-0.15in}
\begin{small}
\begin{align*}
&= \EE_{P^{\pi_b}_\xi}\left[\left(\frac{\pi_e(a|x)}{\widehat\pi_b(a|x)}\left(\R(x,a) - Q(x,a)\right)\right)^2\right] + \EE_{P^{\pi_e}_\xi}\left[\frac{\pi_e(a|x)}{\pi_b(a|x)}\left(\frac{\pi_b(a|x)}{\widehat\pi_b(a|x)} \Delta(x,a)\right)^2\right] - \left(\mathbb{E}_{P^{\pi_e}_\xi}\left[Q(x,a)+\delta(x,a)\Delta(x,a)\right]\right)^2 \\
&+ \EE_{P_0}\left[\left(\EE_{\pi_e}\left[Q(x,a)+\Delta(x,a)\right]\right)^2\right] - 2\EE_{P^{\pi_e}_\xi}\left[\left(1-\delta(x,a)\right)\Delta(x,a)\EE_{\pi_e}\left[Q(x,a)+\Delta(x,a)\right]\right] \\
&\stackrel{\text{(b)}}{=} \EE_{P^{\pi_b}_\xi}\left[\left(\frac{\pi_e(a|x)}{\widehat\pi_b(a|x)}\left(\R(x,a) - Q(x,a)\right)\right)^2\right] + \EE_{P^{\pi_e}_\xi}\left[\frac{\pi_e(a|x)}{\pi_b(a|x)}\left(\frac{\pi_b(a|x)}{\widehat\pi_b(a|x)} \Delta(x,a)\right)^2\right] - \left(\mathbb{E}_{P^{\pi_e}_\xi}\left[Q(x,a)+\delta(x,a)\Delta(x,a)\right]\right)^2 \\
&+\EE_{P^{\pi_e}_\xi}\left[\left(Q(x,a) + \delta(x,a)\Delta(x,a)\right)^2\right] - \EE_{P_0}\left[\left(\EE_{\pi_e}\left[\left(1-\delta(x,a)\right)\Delta(x,a)\right]\right)^2\right] \\
&= \EE_{P^{\pi_b}_\xi}\left[\widehat{\omega}(x,a)^2\big(\R(x,a) - Q(x,a)\big)^2\right] +\mathbb V_{P_0}\left(\EE_{\pi_e}\left[Q(x, a)+\delta(x,a)\Delta(x,a)\right]\right) +\EE_{P_0,\pi_e}\Big[\omega(x,a)\Big(\big(1-\delta(x,a)\big)\Delta(x,a)\Big)^2 \\ 
&- \Big(\EE_{\pi_e}\big[\big(1-\delta(x,a)\big)\Delta(x,a)\big]\Big)^2\Big],
\end{align*}
\end{small}
\vspace{-0.15in}

which concludes the proof.

{\bf (a)} Here we used the fact that 

\vspace{-0.15in}
\begin{small}
\begin{equation*}
\EE_{P^{\pi_b}_\xi}\left[\frac{\pi_e(a|x)}{\widehat\pi_b(a|x)}\Delta(x,a)\left(r(x,a)-Q(x,a)\right)\right] = \EE_{P_0,\pi_b}\bigg[\Delta(x,a)\frac{\pi_e(a|x)}{\widehat\pi_b(a|x)}\overbrace{\mathbb{E}_{P_r}\left[r(x,a)-Q(x,a)\right]}^{0}\mid x,a\bigg] = 0
\end{equation*}
\end{small}
\vspace{-0.15in}

and 

\vspace{-0.15in}
\begin{small}
\begin{equation*}
\EE_{P^{\pi_b}_\xi}\left[\frac{\pi_e(a|x)}{\widehat\pi_b(a|x)}\left(r(x,a)-Q(x,a)\right)\EE_{\pi_e}\left[Q(x,a)+\Delta(x,a)\right]\right] = \EE_{P_0,\pi_b}\bigg[\frac{\pi_e(a|x)}{\widehat\pi_b(a|x)}\EE_{\pi_e}\left[Q(x,a)+\Delta(x,a)\right]\overbrace{\mathbb{E}_{P_r}\left[r(x,a)-Q(x,a)\right]}^{0} \bigg] = 0
\end{equation*}
\end{small}
\vspace{-0.15in}

{\bf (b)} Developing the last two terms, we obtain

\vspace{-0.15in}
\begin{small}
\begin{align*}
\EE_{P_0}&\left[\left(\EE_{\pi_e}\left[Q(x,a)+\Delta(x,a)\right]\right)^2\right] - 2\EE_{P^{\pi_e}_\xi}\left[\left(1-\delta(x,a)\right)\Delta(x,a)\EE_{\pi_e}\left[Q(x,a)+\Delta(x,a)\right]\right] \\
&= \EE_{P_0}\left[\EE_{\pi_e}\left[Q(x,a)\right]^2\right] + \EE_{P_0}\left[\EE_{\pi_e}\left[\Delta(x,a)\right]^2\right] + 2\EE_{P_0}\left[\EE_{\pi_e}\left[Q(x,a)\right]\EE_{\pi_e}\left[\Delta(x,a)\right]\right] \\
&+ 2\EE_{P^{\pi_e}_\xi}\left[\delta(x,a)\Delta(x,a)\EE_{\pi_e}\left[Q(x,a)+\Delta(x,a)\right]\right] - 2\EE_{P^{\pi_e}_\xi}\left[\Delta(x,a)\EE_{\pi_e}\left[Q(x,a)\right]\right] - 2\EE_{P^{\pi_e}_\xi}\left[\Delta(x,a)\EE_{\pi_e}\left[\Delta(x,a)\right]\right] \\
&= \EE_{P_0}\left[\EE_{\pi_e}\left[Q(x,a)\right]^2\right] - \EE_{P_0}\left[\EE_{\pi_e}\left[\Delta(x,a)\right]^2\right] + 2\EE_{P^{\pi_e}_\xi}\left[\delta(x,a)\Delta(x,a)\EE_{\pi_e}\left[Q(x,a)+\Delta(x,a)\right]\right] \\
&= \EE_{P_0}\left[V^{\pi_e}(x)^2\right] - \EE_{P_0}\left[\EE_{\pi_e}\left[\Delta(x,a)\right]^2\right] + 2\EE_{P^{\pi_e}_\xi}\left[\delta(x,a)\Delta(x,a)\EE_{\pi_e}\left[Q(x,a)+\Delta(x,a)\right]\right] \\
&= \EE_{P_0}\left[V^{\pi_e}(x)^2\right] + \EE_{P_0}\left[\EE_{\pi_e}\left[\delta(x,a)\Delta(x,a)\right]^2\right] + 2\EE_{P_0}\left[V^{\pi_e}(x)\mathbb{E}_{\pi_e}\left[\delta(x,a)\Delta(x,a)\right]\right] \\ 
&- \EE_{P_0}\left[\left(\EE_{\pi_e}\left[\delta(x,a)\Delta(x,a)\right] - \EE_{\pi_e}\left[\Delta(x,a)\right]\right)^2\right] \\
&= \EE_{P^{\pi_e}_\xi}\left[\left(Q(x,a) + \delta(x,a)\Delta(x,a)\right)^2\right] - \EE_{P_0}\left[\left(\EE_{\pi_e}\left[\left(1-\delta(x,a)\right)\Delta(x,a)\right]\right)^2\right]
\end{align*}
\end{small}
\end{proof}

Note that when $\pi_e$ is deterministic, the variance proved in Proposition~\ref{lem:tech_result_bias_variance} may be written as 

\vspace{-0.15in}
\begin{small}
\begin{align*}
n\mathbb{V}_{P^{\pi_b}_\xi}(\hat{\rho}_{\text{DR}}^{\pi_e}) &= \EE_{P^{\pi_b}_\xi}\left[\left(\frac{\mathbf 1\left\{a=\pi_e(x)\right\}}{\widehat\pi_b(a|x)}\left(\R(x,a) - Q(x,a)\right)\right)^2\right] + \mathbb V_{P_0}\left[V^{\pi_e}(x)+\delta\left(x,\pi_e(x)\right)\Delta\left(x,\pi_e(x)\right)\right]\\
&+ \EE_{P_0}\left[\frac{1-\pi_b\left(\pi_e(x)|x\right)}{\pi_b\left(\pi_e(x)|x\right)}\cdot\Delta\left(x,\pi_e(x)\right)^2\left(1-\delta\left(x,\pi_e(x)\right)\right)^2\right],
\end{align*}
\end{small}
\vspace{-0.15in}

which is equal to the variance reported in~\citet{Dudik11DR}.



\subsection{Proof of Theorem~\ref{thm:mdr_variance_cb}}


\begin{paragraph}
{\bf Theorem 1.} {\em The variance of the DR estimator~\eqref{eq:DR-bandit} for stochastic $\pi_e$ may be written as the following two forms:

\vspace{-0.15in}
\begin{small}
\begin{align*}
\textbf{\em (i)} \qquad n\VV_{P^{\pi_b}_\xi}(\hat{\rho}^{\pi_e}_{\text{DR}}) &= \mathbb E_{P^{\pi_b}_\xi}\Big[\omega(x,a)\Big(\mathbb{E}_{\pi_e}\left[\omega(x,a')\widehat{Q}(x,a';\beta)^2\right] - \widehat{V}^{\pi_e}(x;\beta)^2 - 2\R(x,a)\big(\omega(x,a)\widehat{Q}(x,a;\beta)-\widehat{V}^{\pi_e}(x;\beta)\big)\Big) \nonumber \\
&+\omega(x,a)^2\R(x,a)^2-\left(\mathbb E_{\pi_e}[\R(x,a)]\right)^2\Big] + \VV_{P_0}\big(\mathbb E_{\pi_e}[\R(x,a)]\big), \\ \nonumber \\
\textbf{\em (ii)} \qquad n\VV_{P^{\pi_b}_\xi}(\hat{\rho}^{\pi_e}_{\text{DR}}) &= \mathbb E_{P^{\pi_b}_\xi}\big[\omega(x,a)q_\beta(x,a,\R)^\top\Omega_{\pi_b}(x)q_\beta(x,a,\R)\big] + C,
\end{align*}
\end{small}
\vspace{-0.15in}

where \begin{small}$\Omega_{\pi_b}(x)=\mathrm{diag}\big[1/\pi_b(a|x)\big]_{a\in\mathcal{A}}-ee^\top$\end{small} with \begin{small}$e=[1,\ldots,1]^\top$\end{small}; \begin{small}$q_\beta(x,a,\R)=D_{\pi_e}(x)\bar{Q}(x;\beta)-\mathbb{I}(a)r$\end{small} is a row vector with \begin{small}$D_{\pi_e}(x)=\mathrm{diag}\big[\pi_e(a|x)\big]_{a\in\mathcal{A}}$\end{small}, row vector \begin{small}$\bar{Q}(x;\beta)=\big[\widehat{Q}(x,a;\beta)\big]_{a\in\mathcal{A}}$\end{small}, and the row vector of indicator functions \begin{small}$\mathbb I(a)=\big[\mathbf 1\{a'=a\}\big]_{a'\in\mathcal A}$\end{small}; and finally 
\[
C=\mathbb V_{P_0}\big(\mathbb E_{\pi_e}[\R(x,a)]\big) -\mathbb E_{P^{\pi_b}_\xi}\left[\left[\mathbb E_{\pi_e}[\R(x,a)]\right]^2\right]+\mathbb E_{P^{\pi_b}_\xi}\left[\left(1+\omega(x,a)-\frac{1}{\pi^2_b(a|x)}\right)\omega(x,a)\R(x,a)^2\right].
\]}
\end{paragraph}

\begin{proof}
In this proof, we first derive the expression in (i) using the variance formulation from Lemma \ref{lem:tech_result_bias_variance} and later deduce the expression in (ii) using the completion of squares.

For the derivation of (i), recall from Proposition \ref{lem:tech_result_bias_variance} that when the behavior policy $\pi_b$ is known, then the policy approximation error $\delta$ vanishes uniformly, and the variance formulation is given by
\[
\small
\begin{split}
n\mathbb{V}_{P^{\pi_b}_\xi}(\hat{\rho}_{\text{DR}}^{\pi_e}) =&\mathbb E_{P^{\pi_b}_\xi}\left[\left(\frac{\pi_e(a|x)}{\pi_b(a|x)}(\R(x,a) - Q(x,a))\right)^2\right]+\mathbb V_{P_0}\left[V^{\pi_e}(x)\right]\\
&+\mathbb E_{P_0}\left[\mathbb E_{\pi_e}\left[\frac{\pi_e(a|x)}{\pi_b(a|x)}(\Delta(x,a))^2\right] -\left(\mathbb E_{\pi_e}[\Delta(x,a)]\right)^2\right].
 \end{split}
\] 
Using the following total variance decomposition rule:
\[
\VV_{P_0,P_r}\left[\mathbb E_{\pi_e}[\R(x,a)] \right]=\mathbb E_{P_0}\left[\mathbb V_{P_r}\left[\mathbb E_{\pi_e}[\R(x,a)]]\right]\right]+\mathbb V_{P_0}\left[\mathbb E_{P_r}\left[\mathbb E_{\pi_e}[\R(x,a)]\right]\right],
\]
we can show that
\[
\begin{split}
\mathbb V_{P_0}\left[V^{\pi_e}(x)\right]=&\mathbb V_{P_0}\left[\mathbb E_{P_r}\left[\mathbb E_{\pi_e}[\R(x,a)]\right]\right]\\
=&\VV_{P_0,P_r}\left[\mathbb E_{\pi_e}[\R(x,a)] \right]-\mathbb E_{P_0}\left[\mathbb V_{P_r}\left[\mathbb E_{\pi_e}[\R(x,a)]]\right]\right]\\
=&\VV_{P_0,P_r}\left[\mathbb E_{\pi_e}[\R(x,a)] \right]-\mathbb E_{P_0,P_r}\left[\left(\mathbb E_{\pi_e}[\R(x,a)]-V^{\pi_e}(x)\right)^2\right],
\end{split}
\]
where the last equality follows from the fact that $\mathbb V_{P_r}\left[\mathbb E_{\pi_e}[\R(x,a)]\right]=\mathbb E_{P_r}\left[\left(\mathbb E_{\pi_e}[\R(x,a)]-V^{\pi_e}(x)\right)^2\right]$,
Utilizing this equality, the variance $n\mathbb{V}_{P^{\pi_b}_\xi}(\hat{\rho}_{\text{DR}}^{\pi_e})$ can be re-written as
\[
\small
\begin{split}
n\mathbb{V}_{P^{\pi_b}_\xi}(\hat{\rho}_{\text{DR}}^{\pi_e})=& \mathbb E_{P^{\pi_b}_\xi}\left[\left(\frac{\pi_e(a|x)}{\pi_b(a|x)} (\R(x,a) - Q(x,a))\right)^2 \right]+ \VV_{P_0,P_r}\left[\mathbb E_{\pi_e}[\R(x,a)] \right] -\mathbb E_{P_0,P_r}\left[(\mathbb E_{\pi_e}[\R(x,a) - Q(x, a)])^2\right]\\
&+\mathbb E_{P_0,\pi_b}\left[\left(\frac{\pi_e(a|x)}{\pi_b(a|x)}\Delta(x, a)\right)^2\right]- \mathbb E_{P_0}\left[\left(\mathbb E_{\pi_b}\left[\frac{\pi_e(a|x)}{\pi_b(a|x)} \Delta(x,a)\right]\right)^2\right],
\end{split}
\]
which is due to the following importance sampling properties:
\[
\mathbb E_{\pi_e}\left[\frac{\pi_e(a|x)}{\pi_b(a|x)}(\Delta(x,a))^2\right] = \mathbb E_{\pi_b}\left[\left(\frac{\pi_e(a|x)}{\pi_b(a|x)}\right)^2(\Delta(x,a))^2\right],
\]
and 
\[
\mathbb E_{\pi_e}[\Delta(x,a)]=\mathbb E_{\pi_b}\left[\frac{\pi_e(a|x)}{\pi_b(a|x)} \Delta(x,a)\right].
\]
Now recall the definition of the model error $ \Delta(x,a)=\widehat Q(x,a)-Q(x,a)$. Since $\mathbb E_{P_r}[r(x,a)]=Q(x,a)$ for every state $x\in\X$ and action $a\in\A$, one can easily show that 
\[
\small
\begin{split}
&\mathbb E_{P^{\pi_b}_\xi}\left[\left(\frac{\pi_e(a|x)}{\pi_b(a|x)} (\R(x,a) - Q(x,a))\right)^2 \right]+\mathbb E_{P_0}\left[\left(\mathbb E_{\pi_b}\left[\frac{\pi_e(a|x)}{\pi_b(a|x)} \Delta(x,a)\right]\right)^2\right]\\
=&\mathbb E_{P^{\pi_b}_\xi}\left[\left(\frac{\pi_e(a|x)}{\pi_b(a|x)} (\R(x,a) - Q(x,a)-\Delta(x, a))\right)^2 \right]\\
=&\mathbb E_{P^{\pi_b}_\xi}\left[\left(\frac{\pi_e(a|x)}{\pi_b(a|x)} (\R(x,a) - \widehat Q(x,a))\right)^2 \right],
\end{split}
\]
because the corresponding cross-term cancels out. By using the same argument, one can also show that
\[
\begin{split}
\mathbb E_{P_0,P_r}\left[(\mathbb E_{\pi_e}[\R(x,a) - Q(x, a)])^2+\left(\mathbb E_{\pi_e}[\Delta(x,a)]\right)^2\right]=&\mathbb E_{P_0,P}\left[(\mathbb E_{\pi_e}[\R(x,a) - Q(x, a)-\Delta(x,a)])^2\right]\\
=&\mathbb E_{P_0,P}\left[(\mathbb E_{\pi_e}[\R(x,a) - \widehat{Q}(x,a;\beta)])^2\right]
\end{split}
\]
Colelctively, the variance expression $n\mathbb{V}_{P^{\pi_b}_\xi}(\hat{\rho}_{\text{DR}}^{\pi_e})$ can be further simplified by the following chain of equalities:
\begin{equation}\label{eq:arg_1}
\small
\begin{split}
n\mathbb{V}_{P^{\pi_b}_\xi}(\hat{\rho}_{\text{DR}}^{\pi_e})=& \VV_{P_0,P}\left[\mathbb E_{\pi_e}[\R(x,a)] \right] + \mathbb E_{P^{\pi_b}_\xi}\left[\left(\frac{\pi_e(a|x)}{\pi_b(a|x)} (\R(x,a) - Q(x,a)-\Delta(x, a))\right)^2 \right]\\
&- \mathbb E_{P_0,P}\left[(\mathbb E_{\pi_e}[\R(x,a) - Q(x, a)-\Delta(x,a)])^2\right]\\
=&\VV_{P_0,P}\left[\mathbb E_{\pi_e}[\R(x,a)] \right] + \mathbb E_{P^{\pi_b}_\xi}\left[\left(\frac{\pi_e(a|x)}{\pi_b(a|x)} (\R(x,a) - \widehat{Q}(x,a;\beta))\right)^2 \right]- \mathbb E_{P_0,P}\left[(\mathbb E_{\pi_e}[\R(x,a) - \widehat{Q}(x,a;\beta)])^2\right]\\
=&\VV_{P_0,P}\left[\mathbb E_{\pi_e}[\R(x,a)] \right] + \mathbb E_{P_0,P}\left[\EE_{\pi_b}\left[\left(\frac{\pi_e(a|x)}{\pi_b(a|x)} (\R(x,a) - \widehat{Q}(x,a;\beta))\right)^2
\right] -(\mathbb E_{\pi_e}[\R(x,a) - \widehat{Q}(x,a;\beta)])^2\right]\\
=&\VV_{P_0,P}\left[\mathbb E_{\pi_e}[\R(x,a)] \right] + \mathbb E_{P_0,P}\left[\EE_{\pi_b}\left[\left(\frac{\pi_e(a|x)}{\pi_b(a|x)} (\R(x,a) - \widehat{Q}(x,a;\beta))\right)^2
\right] -\mathbb E_{\pi_b}\left[\frac{\pi_e(a|x)}{\pi_b(a|x)}\left(\R(x,a) - \widehat{Q}(x,a;\beta)\right)\right]^2\right]\\
=&\VV_{P_0,P}\left[\mathbb E_{\pi_e}[\R(x,a)] \right] + \mathbb E_{P_0,P}\left[\VV_{\pi_b}\left[\frac{\pi_e(a|x)}{\pi_b(a|x)} (\R(x,a) - \widehat{Q}(x,a;\beta))\right]\right]
\end{split}
\end{equation}

To simplify the above variance formulation, one also notices that the conditional variance term can be re-written as
\begin{equation}\label{eq:arg_2}
\small
\begin{split}
& \VV_{\pi_b} \left[\frac{\pi_e(a|x)}{\pi_b(a|x)} (\R(x,a) - \widehat{Q}(x,a;\beta))  \right]\\
=& \mathbb E_{\pi_b} \left[\left(\frac{\pi_e(a|x)}{\pi_b(a|x)} (\R(x,a) - \widehat{Q}(x,a;\beta))\right)^2  \right]-\left(\mathbb E_{\pi_e} \left[\R(x,a) - \widehat{Q}(x,a;\beta) \right]\right)^2 \\
=& \mathbb E_{\pi_b} \left[\left(\frac{\pi_e(a|x)}{\pi_b(a|x)} (\R(x,a) - \widehat{Q}(x,a;\beta))\right)^2  +2\frac{\pi_e(a|x)}{\pi_b(a|x)}\R(x,a)\widehat V^{\pi_e}(x;\beta) \right]-(\R_{\pi_e(x)})^2-(\widehat V^{\pi_e}(x;\beta))^2\\
=&\mathbb E_{\pi_b}\left[\frac{\pi^2_e(a'|x)}{\pi^2_b(a'|x)} \widehat{Q}(x,a';\beta)^2\right]-(\widehat V^{\pi_e}(x;\beta))^2+\mathbb E_{\pi_b}\left[\frac{\pi_e^2(a|x)}{\pi^2_b(a|x)}\R(x,a)^2\right]-\R_{\pi_e(x)}^2\\
&-2\mathbb E_{\pi_b} \left[\frac{\pi_e(a|x)}{\pi_b(a|x)}\R(x,a)\left(\frac{\pi_e(a|x)}{\pi_b(a|x)}\widehat{Q}(x,a;\beta)-\widehat V^{\pi_e}(x;\beta)\right)\right],
\end{split}
\end{equation}
where $\R_{\pi_e(x)}= \mathbb E_{\pi_e}[\R(x,a)]$ is the reward random variable induced by policy $\pi_e$. The second equality is based on expanding the term $(\mathbb E_{\pi_e} [\R(x,a) - \widehat{Q}(x,a;\beta) ])^2$ and recalling that $\widehat V^{\pi_e}(x;\beta)=\mathbb E_{\pi_e}[Q(x,a;\beta)]$, and the third equality is based on expanding the term $\left(\frac{\pi_e(a|x)}{\pi_b(a|x)} (\R(x,a) - \widehat{Q}(x,a;\beta))\right)^2$ and collecting the cross-terms.

Therefore, by combining both the arguments in~\eqref{eq:arg_1} and~\eqref{eq:arg_2}, and recalling that $\omega(x,a)=\pi_e(a|x)/\pi_b(a|x)$, the variance expression $n\mathbb{V}_{P^{\pi_b}_\xi}(\hat{\rho}_{\text{DR}}^{\pi_e})$ becomes
\begin{equation}\label{eq:pf_1}
\small
\begin{split}
&n\mathbb{V}_{P^{\pi_b}_\xi}(\hat{\rho}_{\text{DR}}^{\pi_e})\\
=&\VV_{P_0,P}\left[\mathbb E_{\pi_e}[\R(x,a)] \right] + \mathbb E_{P_0,P}\left[
\mathbb E_{\pi_b}\left[\frac{\pi^2_e(a'|x)}{\pi^2_b(a'|x)} \widehat{Q}(x,a';\beta)^2\right]-(\widehat V^{\pi_e}(x;\beta))^2+\mathbb E_{\pi_b}\left[\frac{\pi_e^2(a|x)}{\pi^2_b(a|x)}\R(x,a)^2\right]-\R_{\pi_e(x)}^2\right.\\
&\left.-2\mathbb E_{\pi_b} \left[\frac{\pi_e(a|x)}{\pi_b(a|x)}\R(x,a)\left(\frac{\pi_e(a|x)}{\pi_b(a|x)}\widehat{Q}(x,a;\beta)-\widehat V^{\pi_e}(x;\beta)\right)\right]\right]\\
=&\mathbb E_{P^{\pi_b}_\xi}\Big[\omega(x,a)\Big(\mathbb{E}_{\pi_e}\left[\omega(x,a')\widehat{Q}(x,a';\beta)^2\right] - \widehat{V}^{\pi_e}(x;\beta)^2 - 2\R(x,a)\big(\omega(x,a)\widehat{Q}(x,a;\beta)-\widehat{V}^{\pi_e}(x;\beta)\big)\Big)\\
&+\omega(x,a)^2\R(x,a)^2-\left(\mathbb E_{\pi_e}[\R(x,a)]\right)^2\Big] + \VV_{P_0}\big(\mathbb E_{\pi_e}[\R(x,a)]\big).
\end{split}
\end{equation}

For the derivation of (ii), consider the parts of the variance objective function that is dependent of $\beta$, i.e.,
\[
\mathbb E_{P^{\pi_b}_\xi}\Big[\omega(x,a)\Big(\mathbb{E}_{\pi_e}\left[\omega(x,a')\widehat{Q}(x,a';\beta)^2\right] - \widehat{V}^{\pi_e}(x;\beta)^2 - 2\R(x,a)\big(\omega(x,a)\widehat{Q}(x,a;\beta)-\widehat{V}^{\pi_e}(x;\beta)\big)\Big)\Big].
\]
From the definition of $|\A|\times|\A|$ symmetric positive-semidefinite matrix $\Omega_{\pi_b}(x)=\mathrm{diag}\big[1/\pi_b(a|x)\big]_{a\in\mathcal{A}}-ee^\top$ and $|\A|\times 1$ vector $q_\beta(x,a,\R)=D_{\pi_e}(x)\bar{Q}(x;\beta)-\mathbb{I}(a)r$,
one can easily sees that
\[
q_\beta(x,a,\R)^\top\Omega_{\pi_b}(x)q_\beta(x,a,\R)=q_\beta(x,a,\R)^\top\cdot\mathrm{diag}\big[1/\pi_b(a|x)\big]_{a\in\mathcal{A}}\cdot q_\beta(x,a,\R) - (q_\beta(x,a,\R)^\top e)^2,
\]
for which
\begin{equation}\label{eq:pf_2}
(q_\beta(x,a,\R)^\top e)^2=\widehat{V}^{\pi_e}(x;\beta)^2+\R(x,a)^2-2\R(x,a)\cdot\widehat{V}^{\pi_e}(x;\beta)
\end{equation}
and
\begin{equation}\label{eq:pf_3}
q_\beta(x,a,\R)^\top\cdot\mathrm{diag}\big[1/\pi_b(a|x)\big]_{a\in\mathcal{A}}\cdot q_\beta(x,a,\R)=\mathbb{E}_{\pi_e}\left[\omega(x,a')\widehat{Q}(x,a';\beta)^2\right]- 2\R(x,a)\omega(x,a)\widehat{Q}(x,a;\beta)+\frac{\R(x,a)^2}{\pi_b(a|x)}.
\end{equation}
By collecting the above terms in~\eqref{eq:pf_1},~\eqref{eq:pf_2}, and~\eqref{eq:pf_3}, one thus can show that
\[
\begin{split}
n&\mathbb{V}_{P^{\pi_b}_\xi}(\hat{\rho}_{\text{DR}}^{\pi_e})=\mathbb E_{P^{\pi_b}_\xi}\Big[\omega(x,a)\cdot q_\beta(x,a,\R)^\top\Omega_{\pi_b}(x)q_\beta(x,a,\R)\Big]\\
&+\mathbb V_{P_0}\big(\mathbb E_{\pi_e}[\R(x,a)]\big) +\mathbb E_{P^{\pi_b}_\xi}\left[\left[\mathbb E_{\pi_e}[\R(x,a)]\right]^2\right]-\mathbb E_{P^{\pi_b}_\xi}\left[\left(1+\omega(x,a)-\frac{1}{\pi^2_b(a|x)}\right)\omega(x,a)\R(x,a)^2\right].
\end{split}
\]
Therefore by recalling the definition of $C$, the derivation of ii) immediately follows.

The proof of this theorem is completed by combining both parts of the above derivations.
\end{proof}


\subsection{Proof of Proposition~\ref{lemma:psd}}

\begin{proposition}
\label{lemma:psd}
For any state $x\in\X$, the symmetric matrix $\Omega_{\pi_b}(x)$ is positive semi-definite (PSD). 
\end{proposition}

\begin{proof}

For any non-zero vector $v\in\mathbb R^{|\mathcal A|}$, consider the quadratic form at a given arbitrary state $x\in\X$: 
\[
v^\top(\mathrm{diag}(\{1/\pi_b(a'|x)\}_{a'\in\mathcal A})-ee^\top)v=\sum_{a\in\mathcal A}{v^2_a}/{\pi_b(a|x)}-(\sum_{a\in\mathcal A}v_a)^2.
\]
Since $\sum_{a\in\mathcal A}\pi_a(a|s)=1$, notice that
\[
\sum_{a\in\mathcal A}{v^2_a}/{\pi_b(a|x)}=\sum_{a\in\mathcal A}\sqrt{\pi_b(a|x)}^2\sum_{a\in\mathcal A}({v_a}/{\sqrt{\pi_b(a|x)}})^2.
\]
Moreover, one deduces that 
\[
\sum_{a\in\mathcal A}\sqrt{\pi_b(a|x)}^2\sum_{a\in\mathcal A}({v_a}/{\sqrt{\pi_b(a|x)}})^2\geq (\sum_{a\in\mathcal A}v_a)^2
\]
from Cauchy-Schwarz inequality.
This implies that $\mathrm{diag}(\{1/\pi_b(a'|x)\}_{a'\in\mathcal A})-ee^\top$ is a positive semi-definite matrix, which completes the proof.
\end{proof}

\newpage
\section{Proofs of Section~\ref{subsec:MRDR-RL}}\label{sec:proofs-MRDR-RL}

\subsection{Proof of Theorem \ref{thm:mdr_variance}} 
We utilize the variance formulation of the DR estimator in the contextual bandit setting and construct the variance of the DR estimator for RL using mathematical induction.

\subsubsection{Base Case, when $T=2$}
Consider the case when $T=2$, where the variance of the DR estimator has the form:
\[
n\cdot\mathbb V_{P^{\pi_b}_\xi}\left[\RV_{0:1}\right]:=\mathbb{V}_{P^{\pi_b}_\xi}\left[\omega_0\left[\R_0+ \gamma\left[\omega_1 \left(\R_1\ - \widehat{Q}^{\pi_e} \left(x_1,a_1;\beta\right)\right) + \widehat{V}^{\pi_e} \left(x_1;\beta\right)\right]-\widehat{Q}^{\pi_e} \left(x_0,a_0;\beta\right)\right]+ \widehat{V}^{\pi_e} \left(x_0;\beta\right) \right].
\]
Using the total variance decomposition rule 
\[
\VV_{P^{\pi_b}_\xi}[\cdot] = \EE_{\mathcal F_0} \left[ \VV_{\mathcal F_1}[\cdot | \mathcal F_0]\right] + \VV_{\mathcal F_0} \left[\EE_{\mathcal F_1}[\cdot | \mathcal F_0] \right],
\]
this formulation can be immediately re-written as
\begin{small}
\begin{align}
\mathbb V_{P^{\pi_b}_\xi}\left[\RV_{0:1}\right]
= 
& \EE_{\mathcal F_0}\! \left[\VV_{\mathcal F_1}\! \left[\omega_0\left[\R_0+ \gamma\left[\omega_1 \left(\R_1\! -\! \widehat{Q}^{\pi_e} \left(x_1,a_1;\beta\right)\right)\! +\! \widehat{V}^{\pi_e} \left(x_1;\beta\right)\right]-\widehat{Q}^{\pi_e} \left(x_0,a_0;\beta\right)\!\right]\!+ \widehat{V}^{\pi_e} \left(x_0;\beta\right) \!\mid\! \mathcal F_0\right]\!\right] + \nonumber\\
&\,\VV_{\mathcal F_0} \! \left[ \EE_{\mathcal F_1}\! \left[\omega_0\left[\R_0+ \gamma\left[\omega_1 \left(\R_1\! -\! \widehat{Q}^{\pi_e} \left(x_1,a_1;\beta\right)\right) \!+\! \widehat{V}^{\pi_e} \left(x_1;\beta\right)\right]-\widehat{Q}^{\pi_e} \left(x_0,a_0;\beta\right)\!\right]\!+ \widehat{V}^{\pi_e} \left(x_0;\beta\right) \!\mid\! \mathcal F_0\right]\!\right]\label{eq:sample_var_DR_T_2}.
\end{align}
\end{small}
Now consider the second term in~\eqref{eq:sample_var_DR_T_2}. By the linearity of expectation operator $\mathbb E_{\F_1}$ and by the following property: 
\[\mathbb E_{\F_1}\left[\widehat Q^{\pi_e} \left(x_1,a_1;\beta\right)\right]=\mathbb E_{x_1}\left[\mathbb E_{a_1\sim\pi_b}\left[\omega_1\widehat Q^{\pi_e} \left(x_1,a_1;\beta\right)\right]\right]=\mathbb E_{x_1}\left[\widehat V^{\pi_e} \left(x_1;\beta\right)\right], 
\]
this term can be re-written as follows:
\begin{equation}\label{eq:var_T_2}
\small
\begin{split}
&\VV_{\mathcal F_0}  \left[ \EE_{\mathcal F_1} \left[\omega_0\left[\R_0+ \gamma\left[\omega_1 \left(\R_1\ - \widehat{Q}^{\pi_e} \left(x_1,a_1;\beta\right)\right) + \widehat{V}^{\pi_e} \left(x_1;\beta\right)\right]-\widehat{Q}^{\pi_e} \left(x_0,a_0;\beta\right)\right]+ \widehat{V}^{\pi_e} \left(x_0;\beta\right) \mid \mathcal F_0\right]\right] \\
= &\VV_{\mathcal F_0}  \left[ \omega_0\left[\R_0+ \gamma\EE_{\mathcal F_1} \left[\omega_1 \left(\R_1\ - \widehat{Q}^{\pi_e} \left(x_1,a_1;\beta\right)\right) + \widehat{V}^{\pi_e} \left(x_1;\beta\right)\mid \mathcal F_0\right]-\widehat{Q}^{\pi_e} \left(x_0,a_0;\beta\right)\right]+ \widehat{V}^{\pi_e} \left(x_0;\beta\right) \right]\\
=& \VV_{\mathcal F_0}  \left[ \omega_0\left[\mathbb E_{\F_1}[\bar{\RV}_{0:1}]-\widehat{Q}^{\pi_e} \left(x_0,a_0;\beta\right)\right]+ \widehat{V}^{\pi_e} \left(x_0;\beta\right) \right].
\end{split}
\end{equation}
For the first term in~\eqref{eq:sample_var_DR_T_2}, by the translational invariance of the variance operator $\VV_{\mathcal F_1}$, one has the following simplification:
\begin{equation}\label{eq:exp_T_2}
\small
\begin{split}
&\EE_{\mathcal F_0} \left[\VV_{\mathcal F_1} \left[\omega_0\left[\R_0+ \gamma\left[\omega_1 \left(\R_1\ - \widehat{Q}^{\pi_e} \left(x_1,a_1;\beta\right)\right) + \widehat{V}^{\pi_e} \left(x_1;\beta\right)\right]-\widehat{Q}^{\pi_e} \left(x_0,a_0;\beta\right)\right]+ \widehat{V}^{\pi_e} \left(x_0;\beta\right) \mid \mathcal F_0\right]\right] \\
=&\EE_{\mathcal F_0} \left[\gamma^2\omega_0^2\cdot\VV_{\mathcal F_1} \left[\omega_1 \left(\R_1\ - \widehat{Q}^{\pi_e} \left(x_1,a_1;\beta\right)\right) + \widehat{V}^{\pi_e} \left(x_1;\beta\right)\mid \mathcal F_0\right]\right].
\end{split}
\end{equation}
Therefore, by combining the above results, one shows the variance formulation in Theorem \ref{thm:mdr_variance} for the base case when $T=2$. 

\subsubsection{General Case by Mathematical Induction}
By the principle of mathematical induction, at step $k=j$ for any arbitrary time step $j$, the variance of the $T-j$ step truncated return variable is assumed to have the following form:
\begin{equation}\label{eq:DR_MI}
n\cdot\mathbb V_{P^{\pi_b}_\xi}\left[\RV_{j:T-1}\right]=\sum_{t=j}^{T-1} \mathbb E_{\F_{j:t-1}}\left[\gamma^{2(t-j)}\omega^2_{j:t-1}\VV_{\F_{t}}  \left[ \omega_{t}\left[\EE_{\F_{t:T-1}}[\bar{\RV}_{t:T-1}]-\widehat{Q}^{\pi_e} \left(x_t,a_{t};\beta\right)\right]\right]+ \widehat{V}^{\pi_e} \left(x_t;\beta\right) \right],
\end{equation}
For the case of $k=j-1$, consider the following variance formulation of the truncated return variable:
\[
\begin{split}
\mathbb V_{P^{\pi_b}_\xi}\left[\RV_{j-1:T-1}\right]=&\mathbb V_{P^{\pi_b}_\xi}\left[\sum_{t=j-1}^{T-1}\gamma^{t-(j-1)}\omega_{j-1:t}\R_t-\sum_{t=j-1}^{T-1}\gamma^{t-(j-1)}\left(\omega_{j-1:t}\widehat{Q}^{\pi_e} \left(x_t,a_t;\beta\right)-\omega_{j-1:t-1}\widehat{V}^{\pi_e} \left(x_t;\beta\right)\right)\right]\\
=&\mathbb{V}_{P^{\pi_b}_\xi}\left[\omega_{j-1}\left[\R_{j-1}+ \gamma\RV_{j:T-1}- \widehat{Q}^{\pi_e} \left(x_{j-1},a_{j-1};\beta\right)\right]+ \widehat{V}^{\pi_e} \left(x_{j-1};\beta\right) \right].
\end{split}
\]
Again, by the total variance decomposition rule 
\[
\VV_{P^{\pi_b}_\xi}[\cdot] = \EE_{\mathcal F_{j-1}} \left[ \VV_{\mathcal F_{j:T-1}}[\cdot | \mathcal F_{j-1}]\right] + \VV_{\mathcal F_{j-1}} \left[\EE_{\mathcal F_{j:T-1}}[\cdot | \mathcal F_{j-1}] \right],
\]
this variance formulation can be re-written using the following chain of equalities:
\begin{small}
\begin{align}
&\mathbb V_{P^{\pi_b}_\xi}\left[\RV_{j-1:T-1}\right]\nonumber\\
= 
& \EE_{\F_{j-1}} \left[\VV_{\F_{j:T-1}} \left[\omega_{j-1}\left[\R_{j-1}+ \gamma\RV_{j:T-1}-\widehat{Q}^{\pi_e} \left(x_{j-1},a_{j-1};\beta\right)\right]+ \widehat{V}^{\pi_e} \left(x_{j-1};\beta\right) \mid \F_{j-1}\right]\right] \nonumber\\
&+\VV_{\F_{j-1}}  \left[ \EE_{\F_{j:T-1}} \left[\omega_{j-1}\left[\R_{j-1}+ \gamma\RV_{j:T-1}-\widehat{Q}^{\pi_e} \left(x_{j-1},a_{j-1};\beta\right)\right]+ \widehat{V}^{\pi_e} \left(x_{j-1};\beta\right) \mid \F_{j-1}\right]\right]\nonumber\\
=&\EE_{\F_{j-1}} \left[\gamma^2\omega_{j-1}^2\cdot\VV_{\F_{j:T-1}} \left[\RV_{j:T-1}\mid \F_{j-1}\right]\right]\nonumber\\
&+\VV_{\F_{j-1}}  \left[ \omega_{j-1}\left[\R_{j-1}+ \gamma\EE_{\F_{j:T-1}} \left[\RV_{j:T-1}\mid \F_{j-1}\right]-\widehat{Q}^{\pi_e} \left(x_{j-1},a_{j-1};\beta\right)\right]+ \widehat{V}^{\pi_e} \left(x_{j-1};\beta\right) \right]\nonumber\\
=&\VV_{\F_{j-1}}  \left[ \omega_{j-1}\left[\mathbb E_{\F_{j:T-1}}[\bar{\RV}_{j-1:T-1}]-\widehat{Q}^{\pi_e} \left(x_{j-1},a_{j-1};\beta\right)\right]+ \widehat{V}^{\pi_e} \left(x_{j-1};\beta\right) \right]+\EE_{\F_{j-1}} \left[\gamma^2\omega_{j-1}^2\cdot\VV_{\F_{j:T-1}} \left[\RV_{j:T-1}\mid \F_{j-1}\right]\right].\nonumber\\
=&\VV_{\F_{j-1}}  \left[ \omega_{j-1}\left[\mathbb E_{\F_{j:T-1}}[\bar{\RV}_{j-1:T-1}]-\widehat{Q}^{\pi_e} \left(x_{j-1},a_{j-1};\beta\right)\right]\right]\nonumber\\
&\quad+\EE_{\F_{j-1}} \left[\gamma^2\omega_{j-1}^2\cdot\sum_{t=j}^{T-1} \mathbb E_{\F_{j:t-1}}\left[\gamma^{2(t-j)}\omega^2_{j:t-1}\VV_{\F_{t}}  \left[ \omega_{t}\left[\mathbb E_{\F_{t+1:T-1}}\left[\bar{\RV}_{t:T-1}\right]-\widehat{Q}^{\pi_e} \left(x_t,a_{t};\beta\right)\right]\right]+ \widehat{V}^{\pi_e} \left(x_t;\beta\right) \right]\right]\nonumber\\
=&\sum_{t=j-1}^{T-1} \mathbb E_{\F_{j-1:t-1}}\left[\gamma^{2(t-(j-1))}\cdot\omega^2_{j-1:t-1}\cdot\VV_{\F_{t}}  \left[ \omega_{t}\left[\mathbb E_{\F_{t+1:T-1}}[\bar{\RV}_{t:T-1}]-\widehat{Q}^{\pi_e} \left(x_t,a_{t};\beta\right)\right]\right]+ \widehat{V}^{\pi_e} \left(x_t;\beta\right) \right].
\end{align}
\end{small}
The second equality is due to the linearity of the expectation operator $\EE_{\F_{j:T-1}}$, and the square-scaling and the translational invariance property of the variance operator, i.e.,
\[
\VV_{\F_{j:T-1}} \left[\omega_{j-1}\left[\R_{j-1}+ \gamma\RV_{j:T-1}-\widehat{Q}^{\pi_e} \left(x_{j-1},a_{j-1};\beta\right)\right]+ \widehat{V}^{\pi_e} \left(x_{j-1};\beta\right) \mid \F_{j-1}\right] =\gamma^2\omega_{j-1}^2\cdot\VV_{\F_{j:T-1}} \left[\RV_{j:T-1}\mid \F_{j-1}\right].
\]
The third equality is due to the fact that
\[
\EE_{\F_{j:T-1}}[\bar{\RV}_{j-1:T-1}]=\R_{j-1}+ \gamma\EE_{\F_{j:T-1}} \left[\RV_{j:T-1}\mid \F_{j-1}\right]
\]
The fourth equality is based on the translational invariance property of $\VV_{\F_{j-1}} $, and the assumption of mathematical induction, and finally the fifth equality is merely based on re-formulating the summation and noticing that $\omega_{t_1:t_2}=1$ if $t_1>t_2$. Therefore, the variance of the truncated return random variable holds in the case when $k=j-1$. 

Based on the argument of induction, one concludes that for any $k\in\{0,\ldots,T-1\}$, the variance of the $T-k$ truncated return in~\eqref{eq:DR_MI} holds. 

To complete the proof of Theorem \ref{thm:mdr_variance}, one needs to show that for any $t\geq 0$, 
\[
\VV_{\F_{t}}  \left[ \omega_{t}\left[\mathbb E_{\F_{t+1:T-1}}[\bar{\RV}_{t:T-1}]-\widehat{Q}^{\pi_e} \left(x_t,a_{t};\beta\right)\right]\right]=\VV_{\F_{t:T-1}}  \left[ \omega_{t}\left[\bar{\RV}_{t:T-1}-\widehat{Q}^{\pi_e} \left(x_t,a_{t};\beta\right)\right]\right]+C_t.
\]
By expanding the variance expression on the left side, by utilizing the completion of squares, and by noticing that $\bar{\RV}_{t:T-1}$ is an unbiased sampling-based estimate of $\F_{t:T-1}[\bar{\RV}_{t:T-1}]$, one has the following chain of equalities:
\[
\begin{split}
&\VV_{\F_{t}}  \left[ \omega_{t}\left[\mathbb E_{\F_{t+1:T-1}}[\bar{\RV}_{t:T-1}]-\widehat{Q}^{\pi_e} \left(x_t,a_{t};\beta\right)\right]\right]\\
=&\EE_{\F_{t}}  \left[\left( \omega_{t}\left[\mathbb E_{\F_{t+1:T-1}}[\bar{\RV}_{t:T-1}]-\widehat{Q}^{\pi_e} \left(x_t,a_{t};\beta\right)\right]\right)^2\right]-\left(\EE_{\F_{t}}  \left[ \omega_{t}\left[\mathbb E_{\F_{t+1:T-1}}[\bar{\RV}_{t:T-1}]-\widehat{Q}^{\pi_e} \left(x_t,a_{t};\beta\right)\right]\right]\right)^2\\
=&\EE_{\F_{t}}  \left[\left( \omega_{t}\left[\mathbb E_{\F_{t+1:T-1}}[\bar{\RV}_{t:T-1}]-\widehat{Q}^{\pi_e} \left(x_t,a_{t};\beta\right)\right]\right)^2\right]-\left(\EE_{\F_{t:T-1}}  \left[ \omega_{t}\left[\bar{\RV}_{t:T-1}-\widehat{Q}^{\pi_e} \left(x_t,a_{t};\beta\right)\right]\right]\right)^2\\
=&\EE_{\F_{t:T-1}}  \left[\left( \omega_{t}\left[\bar{\RV}_{t:T-1}-\widehat{Q}^{\pi_e} \left(x_t,a_{t};\beta\right)\right]\right)^2-2\omega_t^2(\bar{\RV}_{t:T-1}-\widehat{Q}^{\pi_e} \left(x_t,a_{t};\beta\right))\left(\bar{\RV}_{t:T-1}-\mathbb E_{\F_{t+1:T-1}}[\bar{\RV}_{t:T-1}]\right)\right]\\
&+E_{\F_{t:T-1}}\left[\omega_t^2\left(\bar{\RV}_{t:T-1}-\mathbb E_{\F_{t+1:T-1}}[\bar{\RV}_{t:T-1}]\right)^2\right]-\left(\EE_{\F_{t:T-1}}  \left[ \omega_{t}\left[\bar{\RV}_{t:T-1}-\widehat{Q}^{\pi_e} \left(x_t,a_{t};\beta\right)\right]\right]\right)^2\\
=&\EE_{\F_{t:T-1}}  \left[\left( \omega_{t}\left[\bar{\RV}_{t:T-1}-\widehat{Q}^{\pi_e} \left(x_t,a_{t};\beta\right)\right]\right)^2\right]-\left(\EE_{\F_{t:T-1}}  \left[ \omega_{t}\left[\bar{\RV}_{t:T-1}-\widehat{Q}^{\pi_e} \left(x_t,a_{t};\beta\right)\right]\right]\right)^2\\
&+\underbrace{E_{\F_{t:T-1}}\left[\omega_t^2\left(\bar{\RV}_{t:T-1}-\mathbb E_{\F_{t+1:T-1}}[\bar{\RV}_{t:T-1}]\right)^2-2\omega_t^2\bar{\RV}_{t:T-1}\left(\bar{\RV}_{t:T-1}-\mathbb E_{\F_{t+1:T-1}}[\bar{\RV}_{t:T-1}]\right)\right]}_{C_t}\\
=&\VV_{\F_{t:T-1}}  \left[ \omega_{t}\left[\bar{\RV}_{t:T-1}-\widehat{Q}^{\pi_e} \left(x_t,a_{t};\beta\right)\right]\right]+C_t.
\end{split}
\]

Therefore, by combining all the previous arguments, one proves the claim of Theorem \ref{thm:mdr_variance}.

\newpage
\section{Proofs of Section~\ref{subsec:MRDR-asy}}\label{sec:proofs-MRDR-asy}
\begin{proposition}
Suppose the importance weight $\omega_{t_1,t_2}$ is bounded at any time interval $(t_1,t_2)$, such that $0\leq t_1\leq t_2\leq T-1$. Then the MRDR estimator is strongly consistent, i.e., $\lim_{n\rightarrow\infty}\hat{\rho}^{\pi_e}_{\text{MRDR},n}(\beta^*_n)= \rho^{\pi_e}$ almost surely. 
\end{proposition}
\begin{proof}
Consider the general DR estimator in \eqref{eq:DR} for any given model parameter $\beta$. In this proof we use the notation $\widehat\rho^{\pi_e}_n(\beta)$ to emphasize that the sample-size of constructing the DR estimator is $n$. By Lemma 12 of \cite{Thomas16DE}, one first can show that 
\[
\lim_{n\rightarrow\infty}\frac{1}{n} \sum_{i=1}^n\sum_{t=0}^{T-1}\gamma^t\omega_{0:t}^{(i)}\R^{(i)}_t = \rho^{\pi_e} \quad\text{almost surely}.
\]

Now consider the two terms that constitute to the control variate:
\[
X_n=\frac{1}{n} \sum_{i=1}^n\sum_{t=0}^{T-1}\gamma^t\omega_{0:t}^{(i)}\widehat{Q}^{\pi_e}(x^{(i)}_t,a^{(i)}_t;\beta),\quad\quad Y_n=\frac{1}{n} \sum_{i=1}^n\sum_{t=0}^{T-1} \gamma^t\omega_{0:t-1}^{(i)}\widehat{V}^{\pi_e}(x^{(i)}_t;\beta).\nonumber
\]
Again by Lemma 12  in \cite{Thomas16DE}, one can show that,
\[
\lim_{n\rightarrow\infty}X_n=\mathbb E\left[\sum_{t=0}^{T-1}\gamma^t\widehat{Q}^{\pi_e}(x_t,a_t;\beta)\mid x_0,\pi_e\right]\quad\text{almost surely}.
\]
On the other hand, by Lemma 13 in in \cite{Thomas16DE}, one can also show that
\[
\lim_{n\rightarrow\infty}Y_n=\mathbb E\left[\sum_{t=0}^{T-1}\gamma^t\widehat{V}^{\pi_e}(x_t;\beta)\mid x_0,\pi_e\right]\quad\text{almost surely}.
\]
This further implies that with probability $1$, the following chain of equalities hold:
\[
\begin{split}
\lim_{n\rightarrow\infty}Y_n=&\mathbb E\left[\sum_{t=0}^{T-1}\gamma^t\mathbb E_{\pi_e}\left[\widehat{Q}^{\pi_e}(x_t;,a_t\beta)\right]\mid x_0,\pi_e\right]\\
&\mathbb E\left[\sum_{t=0}^{T-1}\gamma^t\widehat{Q}^{\pi_e}(x_t;,a_t\beta)\mid x_0,\pi_e\right].
\end{split}
\]
Therefore, by combining the above results, for any DR estimators we have that $\lim_{n\rightarrow\infty}\widehat\rho^{\pi_e}_{\text{DR},n}(\beta)=\rho^{\pi_e}$ almost surely. This convergence property is point-wise and it holds on any arbitrary choices of model parameters $\beta$. 

Utilizing the above result, and using nboth the triangular inequality and the Slutsky's theorem, one can further show that
\[
\lim_{n\rightarrow\infty}|\widehat\rho^{\pi_e}_{\text{MRDR},n}(\beta^*_n)-\rho^{\pi_e}|\leq \lim_{n\rightarrow\infty}\left|\widehat\rho_{\text{MRDR},n}^{\pi_e}(\beta^*_n)-\widehat\rho^{\pi_e}_{\text{MRDR},n}(\beta^*)\right|+\lim_{n\rightarrow\infty}\left|\widehat\rho^{\pi_e}_{\text{MRDR},n}(\beta^*)-\rho^{\pi_e}\right|=0 + 0,
\]
the convergence result is by using Slutsky's theorem and the fact that $\beta_n\rightarrow\beta^*$, the second convergence result is based on the above consistency result. 
Therefore, the proof of the main claim in this lemma is completed.
\end{proof}

\newpage
\section{Details of Experimental Setup}\label{sec:more-experiments}

\subsection{Contextual Bandit}
Following the setting in~\citet{Dudik11DR}, we evaluate the policy evaluation algorithms in a classification setting. The data comes from standard datasets from UCI repositories. Table~\ref{table:bandit-datasets} shows the statistics of these data-sets. The standard procedure of turning a classification problem into a contextual bandit problem is described as follows.  

{\bf Classification to Contextual Bandits.} Consider a multi-label classification dataset $(x_i, y_i)_{i=1\ldots m}$ where 
$x_i \in \mathbb{R}^d$ is the data,
$y_i \in \{ 1, \ldots, l \}$ is its class,
and $m$ and $l$ are the number of data points and classes respectively. 
A classification algorithm will assign a class to each data point $f: \mathbb{R}^d \to \{ 1, \ldots, l \}$.
If we treat each data point as a context then a deterministic (stochastic) classification algorithm can be viewed as a deterministic (stochastic) policy that maps contexts to actions. 

{\bf Classification Policy.} Assume we have a classification model given by $f: \X\rightarrow\{1,\ldots,l\}$. In the contextual bandit setting, a policy $\pi$ can be constructed from this model when one views the feature $x\in\X$ as the context and the predicted label $\widehat y$ as the action chosen by this policy. During inference, one obtains label prediction $\{\widehat y_i\}_{i=1}^{N_{\text{test}}}$ from the evaluation data-set $\{(x_i,y_i)\}_{i=1}^{N_{\text{test}}}$, and one can construct the set of actions that is executed by the evaluation policy as $a_i=\widehat y_i$ for $i\in\{1,\ldots,N_{\text{test}}\}$.
Correspondingly, at context $x_i$ the agent receives a unit reward, i.e., $r_i=1$, if it accurately predicts the label (that is, $a_i=y_i=\widehat y_i$), and it receives no reward, i.e., $r_i=0$, otherwise.
In this case, one can generate triplets of $\{(x_i, a_i, r_i)\}_{i=1}^{N_{\text{test}}}$ using the evaluation results of the classification model, where the reward of a policy would the accuracy of the corresponding classification algorithm. \cite{Dudik11DR} used the same setup in their evaluation procedure.
Note that applying the evaluation policy to a classification dataset can lead to many bandit datasets because the policy is stochastic and can choose different actions in accordance to the same context This way, we generate $N=500$ different bandit test datasets from a given classification dataset. The accuracy and the significance tests are reported after taking average over these $N$ runs.

Given the above transformation our experiment is set up as follows:
\begin{itemize}
\item The dataset is randomly split into training ($70\%$) for the models involved in estimators and test set ($30\%$) for generating samples for the importance sampling part of the estimators.
\item On the training set we run logistic regression with ordinary least squares fitting to build a classifier.
\item The accuracy of the above classifier on the whole data is measured and serves as ground truth value ($\rho^{\pi_e}$).
\item We soften the baseline deterministic policy (a.k.a. learned classifier) to get a stochastic policy regarded as the evaluation policy ($\pi_e$).
\item We further soften the base policy to form the behavior policy ($\pi_b$).
\item The behavior policy is applied to training dataset and the output bandit data is used to train the model part of estimators.
\item The behavior policy is applied to the test data to get behavior samples used in the importance sampling part of the estimators.
\end{itemize}

{\bf Results}. Table~\ref{table:glass} to Table~\ref{table:letter} show all the results of the contextual bandit experiment. 
Besides all the OPE algorithms mentioned in the main paper, we also include the results from the MRDR0 estimator, which finds the model parameters by naively minimizing the empirical second order moment of the DR estimator. 
Theoretically the methodology of MRDR0 is identical to that of MRDR. However in practice, it turns out MRDR outperforms MRDR0 significantly. Theoretical comparisons of these methods are remained as future work.
In most experiments, the proposed MRDR estimator outperforms its alternatives (statistically) significantly. Especially for medium to large datasets the performance improvement is more significant. We believe that data limitation is the main cause of impeding the performance.  

\begin{table}[H]
\centering
\scriptsize
\caption{Bandit Datasets}
\label{table:bandit-datasets}
\begin{tabular}{cccccccccc}
Dataset & Glass  & Ecoli & Vehicle & Yeast & PageBlok   & OptDigits & SatImage & PenDigits& Letter \\ \midrule \midrule
Classes (\# of actions) &6 & 8 & 4 & 10 & 5 & 10 & 6 & 10 & 26 \\ \midrule
Data (\# of data points) & 214 & 336 & 846 & 1484 & 5473 & 5620 & 6435 & 10992 & 20000
\end{tabular}
\end{table}

\begin{table}[H]
\centering
\scriptsize
\caption{Glass}
\label{table:glass}
\begin{tabular}{cccccccc}
Behavior Policy &DM0& DM& IS & DR & MRDR & DR0 & MRDR0  \\
 \midrule   \midrule
Friendly I&0.4303&0.4153&0.0643&0.0557&\bf 0.0485&0.0562&0.1085 \\ \midrule
Friendly II&0.4709&0.4415&0.0937&0.083&\bf 0.0793&0.0848&0.1584 \\ \midrule
Neutral&0.5607&0.5037&0.2018&0.1928&\bf 0.1739&0.1987&0.3039 \\ \midrule
Adversary I&0.5533&0.4782&0.2238&0.199&0.2082&0.2125&0.384 \\ \midrule
Adversary II&0.5624&0.4947&0.2808&\bf 0.2579&0.2845&0.2711&0.4664
\end{tabular}
\end{table}

\begin{table}[H]
\centering
\scriptsize
\caption{Ecoli}
\label{table:ecoli}
\begin{tabular}{cccccccc}
Behavior Policy &DM0 & DM& IS & DR & MRDR & DR0 & MRDR0  \\
 \midrule   \midrule
Friendly I&0.4275&0.4007&0.0828&0.0705&\bf 0.0683&0.071&0.13\\ \midrule
Friendly II&0.4958&0.4367&0.1114&0.0872&\bf 0.0869&0.0902&0.1891\\ \midrule
Neutral&0.7195&0.5601&0.2258&\bf 0.1768&0.2112&0.208&0.3938\\ \midrule
Adversary I&0.7335&0.5819&0.247&\bf 0.1975&0.4165&0.2317&0.5386\\ \midrule
Adversary II&0.7663&0.6383&0.3099&0.2556&\bf 0.2387&0.2983&0.5055
\end{tabular}
\end{table}

\begin{table}[H]
\centering
\scriptsize
\caption{Vehicle}
\label{table:vehicle}
\begin{tabular}{cccccccc}
Behavior Policy &DM0 &DM & IS & DR & MRDR & DR0 & MRDR0  \\
 \midrule   \midrule
Friendly I&0.3593&0.3273&0.0347&0.0217&\bf 0.0202&0.0224&0.0706 \\\midrule
Friendly II&0.4238&0.3499&0.0517&0.0331&\bf 0.0318&0.0356&0.1013\\ \midrule
Neutral&0.5771&0.4384&0.087&0.0604&\bf 0.0549&0.0722&0.1534\\ \midrule
Adversary I&0.5703&0.405&0.0937&0.0616&\bf 0.0516&0.0769&0.1709\\ \midrule
Adversary II&0.6023&0.405&0.1131&0.0712&\bf 0.0602&0.0952&0.2142
\end{tabular}
\end{table}

\begin{table}[H]
\centering
\scriptsize
\caption{Yeast}
\label{table:yeast}
\begin{tabular}{cccccccc}
Behavior Policy &DM0 & DM& IS & DR & MRDR & DR0 & MRDR0  \\
 \midrule   \midrule
Friendly I&0.2989&0.2772&0.022&0.0165&\bf 0.0148&0.0169&0.041 \\ \midrule
Friendly II&0.3563&0.3083&0.033&0.0256&\bf 0.0231&0.0268&0.057 \\ \midrule
Neutral&0.4996&0.3712&0.0907&0.077&\bf 0.0727&0.0867&0.1531\\ \midrule
Adversary I&0.4948&0.3544&0.1165&0.0932&\bf 0.0853&0.1114&0.1917 \\ \midrule
Adversary II&0.5019&0.3235&0.1424&\bf 0.1098&0.1299&0.1365&0.2398
\end{tabular}
\end{table}

\begin{table}[H]
\centering
\scriptsize
\caption{PageBlock}
\label{table:pageblock}
\begin{tabular}{cccccccc}
Behavior Policy &DM0 &DM & IS & DR & MRDR & DR0 & MRDR0  \\
 \midrule   \midrule
Friendly I&0.7918&0.7793&0.0151& \bf 0.0139&0.0152&0.014&0.0306 \\ \midrule
Friendly II&0.8188&0.7923&0.0224&0.0205& \bf 0.0182&0.0211&0.0358 \\ \midrule
Neutral&0.8311&0.795&0.0314&0.029&\bf 0.0232&0.03&0.0544 \\ \midrule
Adversary I&0.8272&0.7931&0.0316&0.0292&\bf 0.0246&0.0302&0.049 \\ \midrule
Adversary II&0.8465&0.8039&0.0378&0.0349&\bf 0.031&0.0368&0.0666
\end{tabular}
\end{table}

\begin{table}[H]
\centering
\scriptsize
\caption{OptDigits}
\label{table:optdigits}
\begin{tabular}{cccccccc}
Behavior Policy &DM0 & DM& IS & DR & MRDR & DR0 & MRDR0  \\
 \midrule   \midrule
Friendly I &0.387&0.3596& 0.0142&	0.0069&	\bf{0.0056}	&0.0073	&0.0337 \\ 
  \midrule
Friendly II &0.4966&0.4263& 0.0218&	0.0115&	\bf 0.0087 &	0.013	&0.0436 \\
  \midrule
Neutral &0.7787&0.5571&0.0632&	0.0412&	\bf 0.0351	&  0.0562	&0.113 \\
  \midrule
Adversary I &0.798&0.5698& 0.0738&	0.0494 &	\bf 0.0473	& 0.0672 &	0.1283 \\
  \midrule
Adversary II &0.8232&0.5949& 0.0862& \bf	0.061 &	0.0697	& 0.0812	&0.137 
\end{tabular}
\end{table}

\begin{table}[H]
\centering
\scriptsize
\caption{SatImage}
\label{table:satimage}
\begin{tabular}{cccccccc}
Behavior Policy &DM0 &DM & IS & DR & MRDR & DR0 & MRDR0  \\
 \midrule   \midrule
 Friendly I&0.3212&0.2884&0.0128&0.0071&\bf 0.0063&0.0073&0.0248 \\ \midrule
Friendly II &0.4116&0.3328&0.0191&0.0107&\bf 0.0087&0.0119&0.0356 \\ \midrule
 Neutral&0.5984&0.3848&0.0413&0.0246&\bf 0.0186& 0.0335&0.0778 \\ \midrule
Adversary I &0.6243&0.3963&0.0459&0.027&\bf 0.0195&0.0383&0.0871 \\ \midrule
Adversary II &0.6585&0.4093&0.0591&0.0364&\bf 0.0262&0.0521&0.1094 
\end{tabular}
\end{table}

\begin{table}[H]
\centering
\scriptsize
\caption{PenDigits}
\label{table:pendigits}
\begin{tabular}{cccccccc}
Behavior Policy &DM0 &DM & IS & DR & MRDR & DR0 & MRDR0  \\
 \midrule
 \midrule
 Friendly I &0.43&0.4014& 0.0103 & 0.0056 & {\bf 0.0037} & 0.0059 & 0.0202 \\
  \midrule
 Friendly II &0.5315&0.4628& 0.0159 & 0.0092 & {\bf 0.0056} & 0.0194 & 0.0309 \\
  \midrule
 Neutral &0.7632&0.564& 0.0450 & 0.0314 & {\bf 0.0138} & 0.0412 & 0.0882 \\
  \midrule
 Adversary I &0.7828&0.5861& 0.0503 & 0.0366 & {\bf 0.0172} & 0.0472 & 0.0976 \\
  \midrule
 Adversary II &0.7915&0.5641& 0.0646 & 0.0444 & {\bf 0.0188} & 0.0611 & 0.01281 \\
\end{tabular}
\end{table}

\begin{table}[H]
\centering
\scriptsize
\caption{Letter}
\label{table:letter}
\begin{tabular}{cccccccc}
Behavior Policy &DM0 & DM & IS & DR & MRDR & DR0 & MRDR0  \\
 \midrule   \midrule
Friendly I&0.4177&0.392&0.0074&0.0056&\bf 0.0044&0.0057&0.0132\\ 
\midrule
Friendly II&0.4729&0.4146&0.0102&0.0077&\bf 0.0054&0.0083&0.0186\\
\midrule
Neutral&0.6347&0.4713&0.0467&0.0363& \bf 0.0315& 0.0456&0.0855\\
\midrule
Adversary I&0.6379&0.46&0.0587&0.0455&\bf 0.0385&0.0575&0.1082\\
\midrule
Adversary II&0.6421&0.4728&0.0714&0.055&0.0481&0.0703&0.1318
\end{tabular}
\end{table}


\subsection{Reinforcement Learning}
We hereby describe the reinforcement learning domains that are used in the experiments and provide all the results of the OPE experiments. 
Similar to the Bandit case, the accuracy and significance results are averaged over $N=100$ runs with different randomly generated test behavioral trajectories.

{\bf The ModelFail Domain}.
The purpose of this domain is to simulate environments that are not known perfectly. The MDP consists of 4 states, however, the agent can not distinguish between the states. The agent starts from the left most node (node 1); with action $a_1$ it goes to upper middle state, and with action $a_2$ it goes to lower middle state. From these two states with any action it transits to the terminating state (rightmost). If the transition occurs from the upper state, it receives reward $1$ otherwise reward of $-1$ is granted. The horizon is $T=2$, and the problem has a partially observable state (i.e., the agent only sees one state). The evaluation policy selects action $a_1$ and $a_2$ with probabilities $0.88$ and  $0.12$ respectively everywhere. The behavior policy has the opposite behavior.

{\bf The ModelWin Domain}.
This environment is built to simulate the cases where the domain is known perfectly. As Fig.~\ref{fig:environments} shows, it consists of 3 states. The agent begins at state $s_1$ (the middle one). Using action $a_1$, it transits to $s_2$ with probability 0.4 and to $s_3$ with probability $0.6$.
Action $a_2$ leads to the opposite behavior, that is, it goes to $s_2$ with probability $0.6$ and to $s_3$ with probability $0.4$. In state $s_3$, both actions will take the agent back to $s_1$ with probability 1.
If the agent visits $s_2$ and $s_3$ it receives reward $1$ and reward $-1$, respectively. The horizon is set to $T=20$. 
The evaluation policy at $s_1$ will take action $a_1$ with probability $0.73$ and $a_2$ with probability $0.27$. The behavior policy has the opposite behavior. In both states $s_2$ and $s_3$, the agent takes actions uniformly at random.

{\bf The Maze Domain}.
Depicted by Fig.~\ref{fig:environments}, the maze is a $4 \times 4$ gridworld domain with $4$ actions (up-right-down-left), and with deterministic transitions. The horizon is set to $T=100$. The reward is always $0$ except when the agent visits the red and green states. The behavior policy is a hand coded policy that moves quickly to position $(2,4)$. From there by taking the downward action, it has a low probability of transitioning to right until it hits position $(4,4)$. The evaluation policy is near-deterministic version of the behavior policy. Details can be found in~\citet{Thomas16DE}.

{\bf The Mountain Car Domain.}
A car is stuck between two hills in interval $[-0.7, 0.5]$ and the agent should move back and forth to gain enough momentum to reach the top of the right hill. The state space consists of position and velocity and the $3$ discrete actions are accelerating forward and backward and stay-still. The position for the initial state is randomly chosen in $[-0.6, -0.4]$ with velocity equal to $0$.
The horizon is set to $T=250$ with a reward of $-1$ per step. The optimal (deterministic) policy is learned in an on-policy fashion via SARSA~\cite{sutton1998reinforcement}.  Different from the standard linear state-action features used in other experiments, in order to increase feature resolution, the linear feature is time-dependent. Specifically, it is a concatenation of state-action features coming from $10$ discretized values of horizon.

{\bf The Cart Pole Domain.}
This environment is extracted from OpenAI Gym~\cite{brockman2016openai} and has been used in some OPE applications~\cite{hanna2016high}. The state consists of position, angle, speed, and angular speed, in which the position and speed are limited to $[-2.4, 2.4]$ and $[-41.8, 41.8]$, respectively. The possible actions are to move the cart to either left or right, and the reward is 1 for each step before the pendulum falls over. The horizon is set to $T=250$. For the initial state vector, each element is chosen uniformly at random from interval $[-0.05, 0.05]$. The optimal (deterministic) policy is learned in an off-policy fashion using $Q$-learning~\cite{sutton1998reinforcement}. 

The full simulation results of OPE in reinforcement learning are given below. Similar to the contextual bandit setting, we also include the results for the MRDR0 method, whose performance is much worse when compared to that of MRDR.
\begin{table}[H]
\vspace{-0.1in}
\centering
\scriptsize
\caption{ModelFail}
\label{table:modelfail}
\begin{tabular}{cccccccc}
\text{Sample Size} & DM0 & DM &IS& DR& MRDR& DR0& MRDR0 \\ \midrule \midrule
32&0.75394&0.07152&1.37601&0.18461&\bf 0.1698&1.16084&2.71942 \\ \midrule
64&0.75394&0.07152&1.07213&0.1314&\bf 0.11405&0.9046&2.11678 \\ \midrule
128&0.75394&0.07152&0.752&0.09901&\bf 0.08188&0.63571&1.47735 \\ \midrule
256&0.75394&0.07152&0.55955&0.06565&\bf 0.05527&0.47211&1.10451\\ \midrule
512&0.75394&0.07152&0.39533&0.04756&\bf 0.03819&0.33391&0.77813
\end{tabular}
\end{table}
\begin{table}[H]
\vspace{-0.1in}
\centering
\scriptsize
\caption{Modelwin}
\label{table:modelwin}
\begin{tabular}{cccccccc}
\text{Sample Size} &DM0 &DM & IS& DR& MRDR& DR0& MRDR0 \\ \midrule \midrule
32&0.06839&0.06182&0.78452&1.55244&\bf 1.46778&1.51858&1.48273\\ \midrule
64&0.06839&0.06182&1.03207&1.13856&\bf 0.98433&1.40758&1.46993\\ \midrule
128&0.06839&0.06182&0.90166&1.4195&\bf 1.27891&1.52634&1.39799 \\ \midrule
256&0.06839&0.06182&0.78507&1.03575&\bf 0.79849&1.10332&1.03958 \\ \midrule
512&0.06839&0.06182&0.55647&0.89655&0.66791&0.97128&0.87673
\end{tabular}
\end{table}

\begin{table}[H]
\vspace{-0.1in}
\centering
\scriptsize
\caption{$4\times4$ Maze}
\label{table:maze}
\begin{tabular}{cccccccc}
\text{Sample Size} &DM0 &DM& IS& DR& MRDR& DR0& MRDR0 \\ \midrule \midrule
128&1.77598&1.77598&6.68579&0.70465&\bf 0.57042&0.70969&3.19432 \\ \midrule
256&1.77598&1.77598&3.50346&0.69886&\bf 0.58871&0.70211&1.20821 \\ \midrule
512&1.77598&1.77598&2.64257&0.60124&0.58879&0.60338&1.33286 \\ \midrule
1024&1.77598&1.77598&1.45434&0.5201&\bf 0.4666&0.52148&0.84007 \\ \midrule
2048&1.77598&1.77598&0.89668&0.3932&\bf 0.31274&0.39425&0.52836
\end{tabular}
\end{table}
\vspace{-0.1in}

\begin{table}[H]
\vspace{-0.1in}
\centering
\scriptsize
\caption{Mountain Car}
\label{table:mountain_car}
\begin{tabular}{cccccccc}
\text{Sample Size} &DM0&DM& IS& DR& MRDR& DR0& MRDR0 \\ \midrule \midrule	
128&15.18808&17.80368&23.11318&16.14661&\bf 14.96227&19.46953&18.38184 \\ \midrule
256&13.66482&14.62359&14.82684&13.89212&\bf 12.48327&22.80573& 12.92946\\ \midrule
512&16.91604&13.22012&8.26484&8.01421&7.89474&7.96849& 10.58691\\ \midrule
1024&15.60588&10.24318&3.26843&3.03239& 3.1359&9.16269&5.94688 \\ \midrule
2048&13.34371&10.91577&2.50591&2.75933&\bf 2.17138&8.25527&5.52493
\end{tabular}
\end{table}
\vspace{-0.1in}

\begin{table}[H]
\vspace{-0.1in}
\centering
\scriptsize
\caption{Cart Pole}
\label{table:cart_pole}
\begin{tabular}{cccccccc}
\text{Sample Size} &DM0& DM& IS& DR& MRDR& DR0& MRDR0 \\ \midrule \midrule						
32&3.76467&3.92319&1.18213&0.34775&\bf 0.27208 &0.40567&0.5751 \\ \midrule
64&3.81091&3.97312&0.82658&0.27905&\bf 0.2353&0.31494&0.46566 \\ \midrule
128&3.76467&3.92319&0.66174&0.18793&0.16455&0.21232&0.40731 \\ \midrule
256&3.67219&3.82333&0.62042&0.17091&\bf 0.16012 &0.1915&0.40959 \\ \midrule
512&3.65485&3.80461&0.31021&0.08455&\bf 0.079&0.08946&0.20822
\end{tabular}
\end{table}
\vspace{-0.1in}

 
\end{document}